\documentclass{article}
\usepackage[utf8]{inputenc}
\usepackage{textcomp}
\usepackage{booktabs}
\usepackage{xcolor,colortbl}
\definecolor{lavender}{rgb}{0.9, 0.9, 0.98}
\usepackage{thmtools}
\usepackage{thm-restate}
\usepackage{amsmath}
\usepackage{amsthm}
\usepackage{amsfonts}
\usepackage{amssymb}
\usepackage{graphicx}
\usepackage{bbm}
\usepackage{cancel}
\usepackage[linesnumbered,ruled]{algorithm2e}
\SetKwProg{Init}{init}{}{}
\DeclareMathOperator{\E}{\mathbb{E}}
\DeclareMathOperator{\p}{\mathbb{P}}

\newtheorem{defn}{Definition}

\newtheorem{theorem}{Theorem}[section]

\newtheorem{lemma}[theorem]{Lemma}
\newtheorem{claim}[theorem]{Claim}
\newtheorem{proposition}[theorem]{Proposition}

\title{On the tightness of linear relaxation based robustness certification methods}
\author{Cheng Tang \footnote{This is a preliminary version of the paper (author contact info: chengtang48@gmail.com)}}
\date{}

\begin{document}

\maketitle

\begin{abstract}
There has been a rapid development and interest in adversarial training and defenses in the machine learning community in the recent years. 
One line of research focuses on improving the performance and efficiency of adversarial robustness certificates for neural networks \cite{gowal:19, wong_zico:18, raghunathan:18, WengTowardsFC:18, wong:scalable:18,  singh:convex_barrier:19, Huang_etal:19, single-neuron-relax:20, Zhang2020TowardsSA}. 
While each providing a certification to lower (or upper) bound the true distortion under adversarial attacks via relaxation, less studied was the tightness of relaxation.
In this paper, we analyze a family of linear outer approximation based certificate methods via a meta algorithm, IBP-Lin. The aforementioned works often lack quantitative analysis to answer questions such as how does the performance of the certificate method depend on the network configuration and the choice of approximation parameters. Under our framework, we make a first attempt at answering these questions, which reveals that the tightness of linear approximation based certification can depend heavily on the configuration of the trained networks.
\end{abstract}

\section{Introduction}
Given a $m$-layer neural network classifier $f^{1:m}: \mathbb{R}^{d}\rightarrow \mathbb{R}^{d_{m}}$, where $d$ denotes the input data dimension and $d_{m}$ denotes the output classes, for any train or test data point $x\in\mathbb{R}^{d}$, $i^*:= \max_{i \in [d_{m}]} f^{1:m}_i(x)$ gives the predicted output label on $x$. It was observed by previous works (see, e.g., \cite{CROWN:18}) that an adversarial attack problem can be formulated as below:
\begin{eqnarray}
\min_{j\in [d_{m}]} \min_{\delta\in B} f^{1:m}_{i^*}(x + \delta) -  f^{1:m}_{j}(x+\delta)
\end{eqnarray}
where $\delta \in\mathbb{R}^d$ is a perturbation variable and $B\subset\mathbb{R}^d$ is a constraint set. 
If the exact solution to the above minimization problem can be found and the minimal value is positive, we know that the model is robust to any perturbation around $x$ within $B$.

In this paper, we focus on the most commonly used $\ell_{\infty}$-attack, where $B:=\{\delta | \|\delta\|_{\infty}\le \epsilon\}$ for some radius $\epsilon$. 
To understand the robustness of a trained neural network, one would ideally gauge the largest $\epsilon$ for it to be robust at a test point. However, for any fixed $\epsilon$, exactly solving the minimization problem is hard. For ReLU networks, it has been recently established that, as a special case of the property-testing problem, the adversarial attack problem is NP-complete in general \cite{reluplex}.  

 All attack algorithms \cite{fsgm, cw, l-bfs} effectively establish upper bounds on the attack radius $\epsilon$ of a given example: If an adversarial example generated by an attack algorithm flips the label of prediction, then its attack radius is an upper bound. But a failed attack at radius $\epsilon$ does not guarantee that the model is safe at attack radius $\epsilon$. There is a growing body of research dedicated to finding a lower bound to the optimal $\epsilon$, by solving a relaxed version of the minimization problem \cite{gowal:19, wong_zico:18, raghunathan:18, WengTowardsFC:18, wong:scalable:18,  singh:convex_barrier:19, Huang_etal:19, single-neuron-relax:20, Zhang2020TowardsSA}.

However, the gap between the true minimum and a lower bound computed from any of the methods above are currently less understood. In this paper, we analyze the tightness of approximation of a family of linear outer approximation based certificate methods.

\subsection{Related Works} 
%
%
%
In this section, we review recent methods that try to establish a lower bound on the attack radius $\epsilon$, which guarantees that no attack vector $\delta$ can be successful at attacking the model within $\epsilon$-distance. 
\paragraph{Integer program (exact solver)}
In the safety verification community, recent works formulated the input-perturbed network property verification problem, which includes adversarial robustness as a special case, as integer program and developed exact solvers based on the formulation \cite{reluplex, Katz_2017, Tjeng2017VerifyingNN}. However, these solvers are intractable beyond toy networks; it was proved in \cite{reluplex} that the property verification problem of ReLU network is NP-complete.

\paragraph{Linear and SDP relaxations}
For ReLU networks, \cite{WengTowardsFC:18} proposes to use a pair of linear functions to lower and upper bound the network output at each point layer-wise, and then the attack problem at the final output layer can be formulated as a linear program and solved exactly. A very similar algorithm is developed by \cite{wong_zico:18} from the dual LP perspective. 
Follow-up works \cite{Zhang2020TowardsSA, Boopathy:19, CROWN:18, wong:scalable:18} use similar approaches to extend the applicability of the approach to activations beyond ReLU function and to more general networks.
The benefit of the linear relaxation approach is that besides fully-connected layers, the lower bound calculation for linear operators such as pooling and convolution can be easily implemented with built-in functions from standard deep learning libraries.
Another line of work \cite{raghunathan:18} uses an exact quadratic formulation of the ReLU adversarial attack problem, which can be approximately solved by SDP (semi-definite program) relaxation.

\paragraph{Interval bound propagation (IBP)}
A more scalable but coarse relaxation is the ``interval bound propagation'' (IBP) \cite{gowal:19, Huang_etal:19, Zhang2020TowardsSA}, where the layer-wise output range of a feed-forward neural network is outer approximated by product of intervals. The IBP bound trades off tightness of certificate with scalability as compared to the other convex approximation bounds.

\paragraph{Tightness analysis and enhancement of relaxation}
Empirically, \cite{salman:19} observed that ``convexification'' based certification methods, including LP-based methods, may suffer from a relaxation barrier in the sense that even solving the exact convex relaxation problem do not significantly close the gap between the work of \cite{wong_zico:18}, an approximate solution to the convex/linear formulation, and the exact bound.
Since the convex relaxation framework in \cite{salman:19} relies on single-neuron relaxation, \cite{single-neuron-relax:20, singh:convex_barrier:19} proposed to close the gap between the relaxed and the exact solution using a multi-neuron joint relaxation. 
On the other hand, theoretically, it was shown that the quadratic (and exact) formulation of ReLU attack with SDP relaxation in \cite{raghunathan:18}, which is not included in the family of methods studied in \cite{salman:19}, is tight on a single-layer network \cite{tightness-sdp:20}, where the author argues that SDP-relaxation is likely non-tight for a multi-layer network without further analysis.

\paragraph{Our contribution}
In this paper, we study the tightness of linear outer approximation methods, which belongs to the family of convex relaxation methods studied in \cite{salman:19}. Here, we focus on the theoretical analysis of the gap between the linear approximated lower bounds and that of the exact bound; in particular, we analyze the effect of network weights, network dimensions, number of layers, and approximation parameters on the tightness of linear approximation.
%
\begin{table}[t]
\caption{Examples of neural network layers that Algorithm \ref{algo:ibp-lin} is applicable}
\label{tab:applicable_layers}
\centering
\resizebox{\columnwidth}{!}{
\begin{tabular}{cccccc}
\toprule
\multicolumn{1}{c}{Neural function types} 
& \multicolumn{1}{c}{Examples} \\
\midrule
General & {convolution} \\
affine functions & {average pooling} \\
& {fully-connected layers} \\
\midrule
Element-wise & {ReLU}\\
non-decreasing & \\
activation function & {sigmoid (most activation functions)}\\
& {max pooling} \\
\bottomrule
\end{tabular}
}
\end{table}
\subsection{Notation}
\paragraph{General feed-forward networks and sub-networks:} In this paper, we restrict our attention to a general $m$-layer feed-forward neural network: For any $\ell \in \{1, \dots, m-1\}$, its $\ell$-th layer is a composite function $f^{\ell}:  \mathbb{R}^{d_{\ell-1}} \rightarrow \mathbb{R}^{d_{\ell}}$, such that for any input $x \in\mathbb{R}^{d_{\ell-1}}$,
\[
f^{\ell} (x) = g^{\ell} (W^{\ell}x + b^{\ell})
\]
where $g^{\ell}$ is an element-wise non-decreasing activation function and $(W^{\ell}, b^{\ell})$ is a matrix-vector pair that represents a general affine function. Table~\ref{tab:applicable_layers} provides a summary of general affine functions and element-wise non-decreasing activation functions. We will use $f^{\ell-k:\ell}$ to denote the sub-network from layer $\ell-k$ to $\ell$, for any $k \in \{0, 1, \dots, \ell-1\}$, i.e.,
\[
f^{\ell-k:\ell}(x) = f^{\ell} \circ f^{\ell-1} \circ \dots \circ f^{\ell-k}(x),
~~~\mbox{for any~} x\in \mathbb{R}^{d_{\ell-k-1}}
\]
To make the notation consistent, we let $f^{\ell:\ell}(x) = f^{\ell}(x)$.
\paragraph{Last-layer modified network}
Given a network $f^{1:m}$, and any $\ell \in \{1,\dots, m\}$ and $k \in \{0,\dots,\ell-1\}$, we use $\tilde{f}^{\ell-k:\ell}$ to denote the last-layer modified version of the sub-network $f^{\ell-k:\ell}$ as such: For $s \in \{\ell-k, \dots, \ell-1\}$, $\tilde{f}^{s}:=f^{s}$, and at layer $\ell$, $\tilde{f}^{\ell}(x) = W^{\ell}x + b^{\ell}$ (that is, activation at last layer $\ell$ is dropped).
In addition, if $\ell=m$, then the weight matrix $W^m$ is modified based on a test point $x_0\in\mathbb{R}^{d}$: Let $i^*:=\arg\max_i f_i^{1:m}(x_0)$. Then $\tilde{f}^m(x)=\tilde{W}^m x + \tilde{b}^m$, where
\[
\tilde{W}^m_i = W^m_{i^*} - W^m_{i}, \mbox{~for all~} i \ne i^*, i \in [d_m]
\]
\[
\tilde{b}^m_i = b^m_{i^*} - b^m_i, \mbox{~for all~} i \ne i^*, i \in [d_m]
\]
\paragraph{$\ell_{\infty}$-norm attack score}
The $\ell_{\infty}$-norm attack problem for a fixed input $x_0$ can be directly formulated as
\begin{equation}
\label{eqn:l_infty_attack}
s:=\min_j \min_{\|\delta\|_{\infty} \le \epsilon} \tilde{f}_j^{1:m} (x_0 + \delta)
\end{equation}
We refer to $s$ as the optimal attack score. The objective of most robustness certification methods is to provide an estimate $\hat{s}$, that lower bounds the optimal score $s$. If $\hat{s}$ is positive, the model is $\epsilon$-robust at $x_0$ and this particular $\epsilon$ is a lower bound to the ``optimal'' attack radius. 

\paragraph{General notation} 
We use $W_i$ to denote the $i$-th row of matrix $W$ and $W_{\star, i}$ to denote its $i$-th column; we use $b_i$ to denote the $i$-th element of vector $b$. We use superscript $\ell$ to denote network parameters or functions at the $\ell$-th layer, e.g., $(g^{\ell}, W^{\ell}, b^{\ell})$ and $f^{\ell}$; we use superscript $\ell-k:\ell$ to denote parameters or functions from $\ell-k$ to $\ell$-th layer, e.g., $f^{\ell-k:\ell}$ denotes the part of a network starting from $\ell-k$-th layer and ending at $\ell$-th layer.
For two vectors $x, y$, we use $x\succeq y$ to denote element-wise ``greater than''; ``$\preceq$'' is defined similarly.
We use $[\cdot]_{+}$ to denote the element-wise $\max(\cdot, 0)$ operation (i.e., ReLU), and we use $[\cdot]_{-}$ to denote the element-wise $\min(\cdot, 0)$ operation. We use $|\cdot|$ to denote the element-wise absolute value operation.
We use $\|\cdot\|_2$ to denote vector $\ell_2$-norm and $\|\cdot\|_{F}$ to denote matrix Frobenius norm.
We use $[n]$ to denote the set of integers from 1 up to $n$.
\section{IBP-Lin: A meta-algorithm for analyzing linear robustness certificates}
In this section, we introduce IBP-Lin, a meta-algorithm that interpolates between IBP \cite{gowal:19} and linear outer approximation based methods \cite{Zhang2020TowardsSA, CROWN:18, WengTowardsFC:18}. We start by reviewing IBP, an empirically efficient robustness certificate method: Since exactly solving the optimization problem in Eq~\eqref{eqn:l_infty_attack} is difficult, a layer-wise constraint set propagation approach is adopted by IBP, where the constraint set is fixed to be a hyper-rectangle.
\begin{defn}[Hyper-rectangle]
\label{defn:hyperrectangle}
Let $\alpha_i \le \beta_i$, for all $i \in [d]$. A hyper-rectangle in $\mathbb{R}^d$ is defined by the Cartesian product of intervals, i.e., $\Pi_{i=1}^d [\alpha_i, \beta_i]$.
\end{defn}
At the $\ell$-th layer of a feed-forward network, input from $\ell-1$-th layer can be outer-approximated by a hyper-rectangular constraint, represented by the upper and lower limits of its axes $(\alpha^{\ell}, \beta^{\ell})$, where $\alpha^{\ell}, \beta^{\ell} \in \mathbb{R}^{d_{\ell}}$.

IBP recursively applies two steps: At layer $\ell$, given a hyper-rectangular bound $(\alpha^{\ell-1}, \beta^{\ell-1})$ from previous layer, it computes a pre-activation IBP bound $(\hat\alpha^{\ell}, \hat\beta^{\ell})$ that represents a hyper-rectangle containing the pre-activation polytope (induced by an affine transformation of the input hyper-rectangle: $W^{\ell}x+b^{\ell}$), where
\begin{eqnarray}
\label{eqn:ibp_lo}
\hat\alpha_i^{\ell}:=\min_{\alpha^{\ell-1} \preceq x \preceq \beta^{\ell-1}} W_i^{\ell} x + b_i^{\ell}
=
b_i^{\ell} + W_i^{\ell} \beta^{\ell-1} + [W_i^{\ell}]_{+} (\alpha^{\ell-1} - \beta^{\ell-1})
\end{eqnarray}
and
\begin{eqnarray}
\label{eqn:ibp_hi}
\hat\beta_i^{\ell}:=\max_{\alpha^{\ell-1} \preceq x \preceq \beta^{\ell-1}} W_i^{\ell} x + b_i^{\ell}
=
b_i^{\ell} + W_i^{\ell} \beta^{\ell-1}  + [-W_i^{\ell}]_{+} (\beta^{\ell-1} - \alpha^{\ell-1})
\end{eqnarray}
Then it applies the element-wise activation function to the pre-activation bound to obtain the IBP bound at layer $\ell$:
\begin{equation}
    \label{eqn:ibp_action}
    \alpha^{\ell}:=g^{\ell}(\hat\alpha^{\ell})
    ~~\mbox{and}~~
    \beta^{\ell}:=g^{\ell}(\hat\beta^{\ell})
\end{equation}
By the assumption that $g^{\ell}$ is non-decreasing, the hyper-rectangle $(\alpha^{\ell}, \beta^{\ell})$ is guaranteed to contain the output of network at the $\ell$-th layer. 
Exact solutions to Eq~\eqref{eqn:ibp_lo} and \eqref{eqn:ibp_hi} have different forms in the literature. In Appendix \ref{appdendix:ibp_derivation}, we show that our form of pre-activation IBP bound is equivalent to the original IBP bound in \cite{gowal:19}.
\subsection{From IBP to IBP-Lin}
One way to view linear outer approximation methods \cite{Zhang2020TowardsSA, CROWN:18, WengTowardsFC:18} is that, we may tighten the IBP hyper-rectangular bound at layer $\ell$ with an alternative (or additional) hyper-rectangular constraint, $(\alpha^{\prime}, \beta^{\prime})$. 
\begin{eqnarray}
\label{dual_bound_tight}
\min_{x} W_i^{\ell}x + b_i^{\ell} \nonumber\\
\mbox{s.t.~~}\alpha^{\ell-1} \preceq x \preceq \beta^{\ell-1}
\mbox{~~and~~}
\alpha^{\prime} \preceq x \preceq \beta^{\prime}
\end{eqnarray}
The enhanced constraint is equivalent to
$$
\max(\alpha^{\ell-1}, \alpha^{\prime}) \preceq x \preceq \min(\beta^{\ell-1}, \beta^{\prime})
$$
Existing linear outer approximation techniques \cite{Zhang2020TowardsSA, CROWN:18, WengTowardsFC:18} follow a similar template to systematically derive an alternative hyper-rectangular constraint $(\alpha^{\prime}, \beta^{\prime})$, built on top of a layer-wise single-layer linear outer approximation, which is summarized in Definition~\ref{defn:single-layer-relax}. 
\begin{defn}[Single-layer linear approximation of neural functions]
\label{defn:single-layer-relax}
Let $f^{\ell}(x)$ be the $\ell$-th layer function in a feed-forward network. 
Let $(\hat\alpha^{\ell}, \hat\beta^{\ell})$ be any two vectors such that any input $x$ to $f^{\ell}$ satisfies the following
\begin{equation}
\label{eqn:pre-act-crit}
\hat\alpha^{\ell} \preceq \hat{z}^{\ell}(x):=W^{\ell}x+b^{\ell} \preceq \hat\beta^{\ell}
\end{equation}
We say $(\hat\alpha^{\ell}, \hat\beta^{\ell})$ is the \textbf{pre-activation lower/upper bound} at layer $\ell$. 
In addition, we say $f^{\ell}(x)$ has a \textbf{linear outer approximation}, if there exists two affine functions $T_L^{\ell}$ and $T_U^{\ell}$ of form
$$
T_L^{\ell}(x) := D_L^{\ell} \hat{z}^{\ell}(x) + b_L^{\ell}
\mbox{~~and~~}
T_U^{\ell}(x) := D_U^{\ell}\hat{z}^{\ell}(x) + b_U^{\ell}
$$
with the following holds
$$
\forall x, ~~s.t.~~ \hat\alpha^{\ell} \preceq \hat{z}^{\ell}(x) \preceq \hat\beta^{\ell},~
T_L^{\ell}(x) \preceq f^{\ell}(x) \preceq T_U^{\ell}(x)
$$
where $D_L^{\ell}, D_U^{\ell}$ are diagonal matrices and $b_L^{\ell}, b_U^{\ell}$ are vectors. 
\end{defn}
The pre-activation bounds $(\hat\alpha^{\ell}, \hat\beta^{\ell})$, the basis for constructing a single-layer linear outer approximation, can be any pair of bounds that guarantees Eq~\eqref{eqn:pre-act-crit}. For example, they can be IBP pre-activation bounds or pre-computed linear outer approximation bounds.
Based on the single-layer relaxation, a multi-layer linear outer approximation is defined as below.
%
\begin{defn}[Multi-layer linear approximation operator for sub-network $\tilde{f}^{\ell-k:\ell}$]
\label{defn:multi-layer-lin-appx}
For any $\ell$ and $k < \ell$. If a network admits single-layer linear approximation from layer $\ell-k$ to layer $\ell-1$. Let $\tilde{f}^{\ell-k:\ell}$ be a last-layer modified version of $f^{\ell-k:\ell}$.
We define a $k$-layer linear outer approximation operator to $\tilde{f}^{\ell-k:\ell}(x)$ as
$$
T^{\ell-k:\ell}_L(x) := W_L^{\ell-k:\ell}(x) + b_L^{\ell-k:\ell}
$$
$$
T^{\ell-k:\ell}_U(x) := W_U^{\ell-k:\ell}(x) + b_U^{\ell-k:\ell}
$$
where
$$
W_L^{\ell:\ell} = W_U^{\ell:\ell} := W^{\ell} \, ,
\mbox{~~}
b_L^{\ell:\ell} = b_U^{\ell:\ell} := b^{\ell}
$$
And for $s = 1, \dots, k-1,$
$$
W_L^{\ell-s:\ell}:= P_L^{\ell-s+1}W^{\ell-s}
\mbox{~and~}
W_U^{\ell-s:\ell}:= P_U^{\ell-s+1}W^{\ell-s}
$$
with
$$
P_L^{\ell-s+1}:=W_L^{\ell-s+1:\ell}D_U^{\ell-s} + [W_L^{\ell-s+1:\ell}]_{+}(D_L^{\ell-s} - D_U^{\ell-s})
$$
$$
P_U^{\ell-s+1}:=W_U^{\ell-s+1:\ell}D_U^{\ell-s} + [-W_U^{\ell-s+1:\ell}]_{+}(D_U^{\ell-s} - D_L^{\ell-s})
$$
and
\begin{eqnarray*}
b_L^{\ell-s:\ell}:=P_L^{\ell-s+1}b^{\ell-s} 
+ [W_L^{\ell-s+1:\ell}]_{-}b_U^{\ell-s} 
+ [W_L^{\ell-s+1:\ell}]_{+}b_L^{\ell-s} + b_L^{\ell-s+1:\ell}
\end{eqnarray*}
\begin{eqnarray*}
b_U^{\ell-s:\ell}:=P_U^{\ell-s+1}b^{\ell-s} 
+[W_U^{\ell-s+1:\ell}]_{+}b_U^{\ell-s} 
+[W_U^{\ell-s+1:\ell}]_{-}b_L^{\ell-s} + b_U^{\ell-s+1:\ell}
\end{eqnarray*}
\end{defn}
LinApprox (Algorithm \ref{algo:lin-appx}) is a general scheme for linear outer approximation of networks (sub-networks) using Definition~\ref{defn:multi-layer-lin-appx}.
Its input constraint $(\alpha^{\ell-k-1}, \beta^{\ell-k-1})$ specifies the input hyper-rectangle to sub-network $\tilde{f}^{\ell-k:\ell}$; 
the pre-activation bounds $(\widehat\alpha^s, \widehat\beta^s)_{s=\ell-k}^{\ell}$ are pre-computed to ensure for any $\alpha^{\ell-k-1}\preceq x\preceq \beta^{\ell-k-1}$, at layer $s \in \{\ell-k, \dots, \ell\}$, 
\[
\widehat\alpha^s\preceq\tilde{f}^{\ell-k:s}(x)\preceq \widehat\beta^s.
\]
The flexibility of LinApprox comes from the choice of pre-activation bounds $(\widehat\alpha^s, \widehat\beta^s)_{s=\ell-k}^{\ell}$ and the specific single-layer linear approximation method.

The multi-layer linear-outer approximation (LinApprox) augmented bound propagation algorithm, IBP-Lin, is presented in Algorithm \ref{algo:ibp-lin}. Corollary~\ref{coro:correctness} shows that IBP-Lin (Algorithm \ref{algo:ibp-lin}) returns an element-wise lower bound vector $\alpha^m$ to $\min_{\alpha^0\preceq x\preceq \beta^0}\tilde{f}^{1:m}(x)_i$, for all $i\in[d_{m}]$.
\begin{algorithm}[t]
\caption{IBP-Lin}
\label{algo:ibp-lin}
    \SetKwInOut{Input}{Input}
    \SetKwInOut{Output}{Output}
   \Input{
   $\tilde{f}^{1:m}$;
   $\epsilon > 0$, $x_0 \in \mathbb{R}^{d}$; linear approximation schedule vector $s \in \mathbb{N}^m$
   }
    \Output{Robustness lower bound vector $\alpha^m$}
    \Init{}{
    	$\alpha^0 \leftarrow x_0 - \epsilon \mathbbm{1}$ \\
	$\beta^0 \leftarrow x_0 + \epsilon \mathbbm{1}$
    }
   \For{$\ell \in 1, \dots, m$}{
	$\widehat{\alpha}^{\ell} \leftarrow W^{\ell}\beta^{\ell-1} + b^{\ell} + [W^{\ell}]_{+} (\alpha^{\ell-1} - \beta^{\ell-1})$ \\
	$\widehat{\beta}^{\ell} \leftarrow W^{\ell}\beta^{\ell-1} + b^{\ell} + [-W^{\ell}]_{+} (\beta^{\ell-1} - \alpha^{\ell-1})$ \\
	\If {$s_{\ell} \ne 0$}{
		Get number of backward approximation layers $k \leftarrow s_{\ell}$  \\
		Get sub-network function $\tilde f^{\ell-k+1:\ell}(x) := W^{\ell}f^{\ell-k+1:\ell-1}(x) + b^{\ell}$
		$(\widehat\alpha^{\prime}, \widehat\beta^{\prime})\leftarrow LinApprox((\alpha^{\ell-k}, \beta^{\ell-k}), \tilde{f}^{\ell-k+1:\ell}, (\hat\alpha^s,\hat\beta^s)_{s=\ell-k}^{\ell})$;\\
		(Re-estimate pre-activation bounds at layer $\ell$ by \textbf{Algorithm~\ref{algo:lin-appx}}) \\
		$\widehat{\alpha}^{\ell} \leftarrow \max(\widehat{\alpha}^{\ell}, \widehat\alpha^{\prime})$; \\
		$\widehat{\beta}^{\ell} \leftarrow \min(\widehat{\beta}^{\ell}, \widehat\beta^{\prime})$; \\
	}
	$\alpha^{\ell} \leftarrow g^{\ell}(\widehat{\alpha}^{\ell})$ 
		and $\beta^{\ell} \leftarrow g^{\ell}(\widehat{\beta}^{\ell})$; \\
   }
   \KwRet\ $\alpha^m$;
\end{algorithm}
\begin{algorithm}[t]
\caption{LinApprox}
\label{algo:lin-appx}
    \SetKwInOut{Input}{Input}
    \SetKwInOut{Output}{Output}
   \Input{
   	Initial constraints $(\alpha^{\ell-k-1}, \beta^{\ell-k-1})$; 
	Sub-network function $\tilde{f}^{\ell-k:\ell}$ with pre-activation bounds $(\hat\alpha^s, \hat\beta^s)_{s=\ell-k}^{\ell}$;
	Single-layer linear approximation method
   }
    \Output{$\hat\alpha^{\prime}, \hat\beta^{\prime}$}
    \Init{}{
    \begin{small}
    	$W_L^{\ell:\ell}\leftarrow W^\ell$ and $W_U^{\ell:\ell} \leftarrow W^{\ell}$; ~~ 
	$b_L^{\ell:\ell}\leftarrow b^{\ell}$ and $b_U^{\ell:\ell}\leftarrow b^{\ell}$; \\
    \end{small}
    }
   \For{$s \in \{1, \dots, k\}$}{
   \begin{small}
   Construct $D_L^{\ell-s}, D_U^{\ell-s}, b_L^{\ell-s}, b_U^{\ell-s}$ with the provided single-layer linear approximation method (Definition~\ref{defn:single-layer-relax}) based on pre-activation bounds $(\hat\alpha^{\ell-s},\hat\beta^{\ell-s})$ ; \\
   
   Update $W_L^{\ell-s:\ell}, W_U^{\ell-s:\ell}, b_L^{\ell-s:\ell}, b_U^{\ell-s:\ell}$ based on the iterative rule in Definition~\ref{defn:multi-layer-lin-appx} ;
	\end{small}
   }
   \begin{small}
   $\hat\alpha^{\prime} \leftarrow W_L^{\ell-k:\ell}\beta^{\ell-k-1} 
   + [W_L^{\ell-k:\ell}]_{+} (\alpha^{\ell-k-1}-\beta^{\ell-k-1}) + b_L^{\ell-k:\ell}$ ; \\
   $\hat\beta^{\prime} \leftarrow W_U^{\ell-k:\ell}\beta^{\ell-k-1} 
   + [-W_U^{\ell-k:\ell}]_{+} (\beta^{\ell-k-1}-\alpha^{\ell-k-1}) + b_U^{\ell-k:\ell}$ ; \\
   \end{small}
   \KwRet $(\hat\alpha^{\prime}, \hat\beta^{\prime})$\;
\end{algorithm}
\begin{restatable}[Correctness of IBP-Lin]{corollary}{correctness}
\label{coro:correctness}
Let $f^{1:m}$ be $m$-layer feed-forward network. Suppose there exists single-layer linear outer approximation for every layer of the network as defined in Definition~\ref{defn:single-layer-relax}.
Suppose the input to the network $\tilde{f}^{1:m}$ satisfies
\[
\alpha^{0} \preceq x \preceq \beta^{0}
\]
Let $\widehat\alpha^{\ell}, \widehat\beta^{\ell}$ be defined as in line 5 and 6 of Algorithm \ref{algo:ibp-lin}, and let $\alpha^m$ be as defined in line 16 of Algorithm \ref{algo:ibp-lin}.
\begin{enumerate}
    \item For any $\ell\in [m]$, and for any $x$ such that $\alpha^{0} \preceq x \preceq \beta^{0}$,
    \[
    \widehat\alpha^{\ell}\preceq \tilde{f}^{1:\ell}(x) \preceq \widehat\beta^{\ell}
    \]
    \item
    For any $x$ such that $\alpha^{0} \preceq x \preceq \beta^{0}$,
    \[
      \alpha^m  \preceq \tilde{f}^{1:m}(x)
    \]
    
\end{enumerate}
\end{restatable}
The proof of Corollary \ref{coro:correctness} relies on Proposition ~\ref{prop:correctness} (see Appendix), whose main idea is the same as the correctness proof of CROWN \cite[Theorem 3.2]{CROWN:18}. 
\paragraph{Schedule vector in IBP-Lin (Algorithm~\ref{algo:ibp-lin}):}
Given a trained network $\tilde{f}^{1:m}$ (with last-layer modified according to), fix any input data $x_0$, and attack radius $\epsilon > 0$, IBP-Lin is parametrized by the linear approximation schedule vector $s\in \mathbb{N}^m$. The value $s_{\ell}$ specifies the number of backward unrolling used to linearly approximate the network output at layer $\ell$: If $s_{\ell}=0$, no linear approximation is used; if $s_{\ell}=k, k\in \{1,\dots, \ell\}$, then the sub-network $\tilde{f}^{\ell-k:\ell}$ is linearly approximated.

\paragraph{Relation to existing algorithms}
IBP-Lin reduces to IBP \cite{gowal:19} if the schedule vector $s$ is null (all-zero); it reduces to CROWN-IBP style algorithm \cite{Zhang2020TowardsSA} if $s_{\ell}=0, \forall \ell < m$, and $s_m=m$; it reduces to CROWN and Fast-Lin style algorithm \cite{CROWN:18, WengTowardsFC:18} if $s_{\ell}=\ell, \forall \ell \in [m]$.

There are some slight generalization and improvement in IBP-Lin: First, IBP-Lin allows more freedom in choosing the single-layer approximation functions; second, while existing linear approximation methods unfold the network from layer $\ell$ to the first layer, IBP-Lin leaves the number of backward unfolding layers as a parameter (specified by the schedule $s_{\ell}$); lastly, after obtaining the alternative pre-activation bounds $(\hat\alpha^{\prime}, \hat\beta^{\prime})$ from LinApprox, it takes its intersection with the original pre-activation bounds (line 11 and 12), this is to make sure that we never get a worse approximation by the additional computation in LinApprox, as is clear from Eq.~\eqref{dual_bound_tight}.

\section{Main result}
While IBP-Lin allows different choices of linear approximation, intuitively the tightness of a linear outer-approximation will be determined by the ``width'' of the approximation tube around the boundary points. The ``thinner'' the approximation tube is, the better the approximation result should be. 
Before proceeding, we characterize one common single-layer linear outer approximation method as the parallelogram relaxation.
\begin{defn}[Parallelogram relaxation]
\label{defn:parallel-relax}
We say a neural network layer $f^{\ell}(x)$ admits a parallelogram relaxation, if it has a single-layer linear outer approximation (Definition~\ref{defn:single-layer-relax}) with
\[
D:=D^{\ell}_L = D^{\ell}_U 
\]
such that for some pre-activation bounds $(\hat\alpha^{\ell}, \hat\beta^{\ell})$,
$$
\forall x~~s.t.~~ \hat\alpha^{\ell} \preceq \hat{z}(x)=W^{\ell}x+b^{\ell} \preceq \hat\beta^{\ell},~
D\hat{z}(x)+b_L \preceq f^{\ell}(x) \preceq D\hat{z}(x)+b_U
$$
\end{defn}
The parallelogram relaxation is adopted by IBP and Fast-Lin \cite{gowal:19, WengTowardsFC:18}. In \cite{CROWN:18} and \cite{Zhang2020TowardsSA}, another type of of relaxation method is used to minimize the volume of the relaxed area. In our subsequent analysis, we focus on parallelogram relaxation, which captures our intuition that the ``parallelogram width'' of the relaxed region controls the quality of approximation.
Proposition~\ref{thm:parallel-relax-error} uses the layer-wise ``parallelogram width'', $b_U^{\ell} - b_L^{\ell}$, as a parameter to upper bound the gap between the relaxed hyper-rectangular approximation of network output and the tightest hyper-rectangle containing the network output.
\begin{restatable}[Approximation error of multi-layer parallelogram-relaxation]{proposition}{parallel}
\label{thm:parallel-relax-error}
\begin{small}
Let $f^{1:m}$ be a neural network classifier that can be layer-wise outer approximated by parallelogram relaxation from $\ell = 1, \dots, m-1$. For any $\ell, 0< k < \ell$, let $\tilde{f}^{\ell-k:\ell}$ denote the last-layer modified version of sub-network $f^{\ell-k:\ell}$. 
Let $T_L^{\ell-k:\ell}, T_U^{\ell-k:\ell}$ be a multi-layer linear outer approximation of $\tilde{f}^{\ell-k:\ell}$ based on parallelogram relaxation. 
%
%
Then for all $i \in [d_{\ell}]$,
\[
\min_{\alpha^{\ell-k-1} \preceq x \preceq \beta^{\ell-k-1}}T_L^{\ell-k:\ell}(x)_i \ge \min_{\alpha^{\ell-k-1} \preceq x \preceq \beta^{\ell-k-1}} \tilde{f}_i^{\ell-k:\ell} (x)
- 
\sum_{s=0}^{k-1} |W_i^{\ell-s:\ell}| (b^{\ell-s-1}_U - b^{\ell-s-1}_L)
\]
\mbox{~~and~~}
\[
\max_{\alpha^{\ell-k-1} \preceq x \preceq \beta^{\ell-k-1}}T_U^{\ell-k:\ell}(x)_i 
\le
\max_{\alpha^{\ell-k-1} \preceq x \preceq \beta^{\ell-k-1}} \tilde{f}_i^{\ell-k:\ell} (x) +
\sum_{s=0}^{k-1} |W_i^{\ell-s:\ell}| (b^{\ell-s-1}_U - b^{\ell-s-1}_L)
\]
\end{small}
\end{restatable}

For any sub-network $\tilde{f}^{\ell-k:\ell}$, Proposition~\ref{thm:parallel-relax-error} can be used to derive an upper bound of approximation error, starting from an input constraint set at layer $\ell-k-1$:
In the statement of Proposition~\ref{thm:parallel-relax-error}, the difference
\[
\min_{\alpha^{\ell-k-1} \preceq x \preceq \beta^{\ell-k-1}}T_L^{\ell-k:\ell}(x)_i 
- 
\min_{\alpha^{\ell-k-1} \preceq x \preceq \beta^{\ell-k-1}} \tilde{f}_i^{\ell-k:\ell} (x)
\]
represents the gap between the approximated lower bound at dimension $i$ and the optimal lower bound induced by the tightest hyper-rectangle containing $\tilde{f}^{\ell-k:\ell} (x)$. The difference is non-positive due to the correctness of IBP-Lin. The larger the quantity $\sum_{s=0}^{k-1} |W_i^{\ell-s:\ell}| (b^{\ell-s-1}_U - b^{\ell-s-1}_L)$, the larger the approximation error. Similar interpretation applies to the other direction.
\subsection{Approximation error of different types of ReLU networks}
For the rest of the analysis, we restrict our attention to parallelogram relaxation of ReLU networks. 
Let $f^{\ell}(x)$ be a ReLU layer function, it admits a single-layer parallelogram approximation whose approximation parameters $D^{\ell}, b_L^{\ell}, b_U^{\ell}$, given pre-activation $(\hat\alpha^{\ell}, \hat\beta^{\ell})$, are defined as
\begin{eqnarray}
\label{defn:relu_approx}
D_{i, i}^{\ell}
=
\begin{cases}
1, \mbox{~if}~ i\in \mathcal{I}^{\ell}_{+} \\
0, \mbox{~if}~ i\in \mathcal{I}^{\ell}_{-} \\
\frac{\hat{\beta}_i^{\ell}}{\hat{\beta}_i^{\ell} - \hat{\alpha}_i^{\ell}}, \mbox{~if}~ i\in \mathcal{I}^{\ell}
\end{cases}
b_L^{\ell} = 0,
~~~\mbox{and~~}
b_U^{\ell} = 
\begin{cases}
0, \mbox{~if}~ i\in \mathcal{I}^{\ell}_{+} \\
0, \mbox{~if}~ i\in \mathcal{I}^{\ell}_{-} \\
\frac{-\hat{\alpha}_i^{\ell}\hat{\beta}_i^{\ell}}{\hat{\beta}_i^{\ell} - \hat{\alpha}_i^{\ell}}, \mbox{~if}~ i\in \mathcal{I}^{\ell}
\end{cases}
\end{eqnarray}
where~~
\[
\begin{cases}
\mathcal{I}^{\ell}_{+}:=\{i | 0 \le\hat{\alpha}^{\ell}_i \le \hat{\beta}^{\ell}_i \} \\
\mathcal{I}^{\ell}:=\{i | \hat{\alpha}^{\ell}_i < 0 < \hat{\beta}^{\ell}_i\} \\
\mathcal{I}^{\ell}_{-}:=\{i | \hat{\alpha}^{\ell}_i \le \hat{\beta}^{\ell}_i \le 0\}
\end{cases}
\]
The definition of subsets $\mathcal{I}^{\ell}, \mathcal{I}_{+}^{\ell}, \mathcal{I}_{-}^{\ell}$, which partitions $[d_{\ell}]$, is a common approach used in linear approximation methods \cite{Zhang2020TowardsSA, CROWN:18, WengTowardsFC:18, wong_zico:18} to refine approximation based on pre-activation bounds.

In \cite{tightness-sdp:20}, the author showed that SDP-relaxation is tight on a single hidden-layer ReLU network (independent of weight assignment), but argued that for multi-layer networks SDP-relaxation is likely non-tight without a formal analysis. 
For single hidden-layer ReLU networks, it is not hard to see that with parallelogram approximation, the upper bound is always tight. We show for random Gaussian input, the lower bound is non-tight in expectation.
%

\begin{restatable}[Lower bound on approximation error of single hidden-layer random networks]{proposition}{singlelayer}
\label{prop:single_layer_lower_bound}
For any weight matrix $W^1$, let $[W^1x]_{+}$ be a single-hidden layer ReLU network (bias free). Let $D^1W^1x+b_L^1$ and $D^1W^1x+b_U^1$ be the lower/upper bound of the single-hidden layer parallelogram approximation as defined in Eq~\eqref{defn:relu_approx}.
\begin{enumerate}
\item
Let $\alpha^0=x_0-\epsilon\vec{1}, \beta^0=x_0+\epsilon\vec{1}$, with $x_0\sim\mathcal{N}(0, I)$. Then let $\Phi(\cdot)$ denote the cumulative distribution function of standard normal, and for any $i$, let $\kappa_i:=\frac{\|W_i^1\|_1}{\|W_i^1\|_2}$. The following holds
\begin{eqnarray*}
\E_{x_0}\bigg(\min_{\alpha^0 \preceq x \preceq \beta^0} D_i^1W^1x+(b_L^1)_i
-
\min_{\alpha^0 \preceq x \preceq \beta^0} [W_i^1x]_{+}
\bigg) \\
\le 
(\Phi(\epsilon \kappa_i)-1/2)\|W_i^1\|_2 (\frac{1}{\epsilon\kappa_i} - \epsilon\kappa_i) 
\end{eqnarray*}
\item
For any $(\alpha^0, \beta^0)$,
\[
\max_{\alpha^0 \preceq x \preceq \beta^0} D_i^1W^1 x + (b_U^1)_i
=
\max_{\alpha^0 \preceq x \preceq \beta^0} [W_i^1x]_{+}
\]
\end{enumerate}
\end{restatable}
The first statement of Proposition~\ref{prop:single_layer_lower_bound} translates to a lower bound on the expected approximation error at the $i$-th output neuron, that is,
\[
(\Phi(\epsilon \kappa_i)-1/2)\|W_i^1\|_2 (\epsilon\kappa_i - \frac{1}{\epsilon\kappa_i} ) 
\]
which is only meaningful (non-negative) when $\epsilon$ is chosen so that
$
\epsilon > \frac{1}{\kappa_i}
$.
Here $\kappa_i = \frac{\|W_i^1\|_1}{\|W_i^1\|_2} \in [1, \sqrt{d}]$ characterizes the sparsity of $W_i^1$. For fixed $\|W_i^1\|_2$, $\kappa_i$ becomes larger when $W_i^1$ becomes more dense. The lower bound suggests that the approximation error grows with $\epsilon$ (attack radius), $\|W_i^1\|_2$ (norm of row weights), and $\kappa_i$ (dense level of row weights).
To simplify our analysis, for the rest of the paper we make the following assumptions:
\begin{enumerate}
    \item We restrict $(\hat\alpha^{\ell}, \hat\beta^{\ell})$ to IBP pre-activation bounds
    \item We consider an $m$-layer ReLU network $f^{1:m}(x)$ such that $\forall \ell\in [m]$, $W^{\ell}\in\mathbb{R}^{d\times d}$, where $d$ equals to the input dimension.
\end{enumerate} 
We let $T_L^{1:m}, T_U^{1:m}$ ($T_L^{\ell-k:\ell}, T_U^{\ell-k:\ell}$) be a multi-layer linear approximation of the network (sub-network) based on the single-layer relaxation in Eq~\eqref{defn:relu_approx}.

%

%
%
%
While we cannot obtain positive results regarding approximation error of linear approximation methods on arbitrary or random multi-layer networks, Proposition~\ref{prop:tight-bound} below demonstrates that there exists ReLU networks where exact bound can be computed.
\begin{restatable}[Tight bound on non-negative sub-networks]{proposition}{tightbound}
\label{prop:tight-bound} 
Suppose $\forall\ell \in [m]$, $W_{i, j}^{\ell}\ge 0$ and $b_i^{\ell}=0$, for any $i, j\in [d]$.
Let $(\alpha^0, \beta^0)$ be any hyper-rectangular input constraint. 
Then for any $2\le\ell \le m$ and $1 \le k \le \ell-2$, 
\[
\min_{\alpha^{\ell-k-1} \preceq x \preceq \beta^{\ell-k-1}}T_L^{\ell-k:\ell}(x)_i = \min_{\alpha^{\ell-k-1} \preceq x \preceq \beta^{\ell-k-1}} \tilde{f}_i^{\ell-k:\ell}(x)
\]
and
\[
\max_{\alpha^{\ell-k-1} \preceq x \preceq \beta^{\ell-k-1}}T_U^{\ell-k:\ell}(x)_i = \max_{\alpha^{\ell-k-1} \preceq x \preceq \beta^{\ell-k-1}} \tilde{f}_i^{\ell-k:\ell}(x)
\]
\end{restatable}
Note that the result holds only when $\ell-k \ge 2$, i.e., when we approximate a sub-network after the first layer. The proof is simple and uses the fact that the output of ReLU network is non-negative and the property of IBP bounds.
The intuition is that parallelogram approximation becomes tight (lower bound coincides with upper bound) in case for all neuron $i$, $i\in \mathcal{I}_{+}$.
%
%
In contrast to the positive result in Proposition~\ref{prop:tight-bound}, when one uses parallelogram approximation for the entire ReLU network (starting from the first layer), Proposition~\ref{prop:transition} shows that the approximation error can depend heavily on the magnitude of the network weights. 
\begin{restatable}[``Phase transition'' of approximation error can happen]{proposition}{transition}
\label{prop:transition}
Suppose $\forall\ell \in [m]$, $W_{i, j}^{\ell}\ge 0$ and $b_i^{\ell}=0$, for any $i, j\in [d]$. Suppose the input constraint is such that $\alpha^0=x_0-\epsilon\vec{1}, \beta^0=x_0+\epsilon\vec{1}$. Then
\begin{enumerate}
    \item Let $x_0$ be any fixed test point. If $W_{ij}^{\ell} = \frac{1}{d^p}$ for some $p > 1$. Then for any $\delta > 0$, choosing the number of layers $m \ge \frac{1}{p-1}\log_d \frac{\epsilon}{2\delta}$ guarantees that $\forall i \in [d]$,
    \[
    \min_{\alpha^0\preceq x\preceq\beta^0}T_L^{1:m}(x)_i
    \ge
    \min_{\alpha^0\preceq x\preceq\beta^0}\tilde{f}_i^{1:m}(x)
    - \delta
    \]
    \item Suppose $x_0\sim \mathcal{N}(0, I)$ and let $\Phi(\cdot)$ denote the c.d.f. of the standard normal distribution. If $W_{ij}^{\ell} = \frac{1}{d^p}$ for some $p < 1$ and suppose $\epsilon > \sqrt{1/d}$. Then for any $B>0$, choosing the number of layers $m$ to be
    \[
        m > \frac{1}{1-p}\log_d \frac{Bd}{(\Phi(\epsilon \sqrt{d}) - 1/2) (\epsilon d - \frac{1}{\epsilon})}
    \]
guarantees that $\forall i \in [d]$,
    \[
    \E_{x_0} \bigg (\min_{\alpha^0\preceq x\preceq\beta^0}T_L^{1:m}(x)_i
    -
    \min_{\alpha^0\preceq x\preceq\beta^0}\tilde{f}_i^{1:m}(x) \bigg)
    \le
    - B
    \]
\end{enumerate}
\end{restatable}
The intuition behind this result is that the parallelogram approximation is not tight at the first layer, and depending on the magnitude of the network weights, this small error either explodes up or vanishes as the approximated bounds propagate through layers of the network.
Our last result is motivated by the following question: Does the linear outer approximation become more accurate at a fixed layer $\ell$ if we apply LinApprox starting from layer $k_1 < k_2$? This seems to be implicitly assumed to be true in the current literature. However, Proposition~\ref{prop:unrolling-hurts} suggests that this is not necessarily the case.
\begin{restatable}[Using more layers in LinApprox can hurt approximation]{proposition}{unrolling}
\label{prop:unrolling-hurts}
For any $B > 0$, there exists an $m$-layer ReLU network and choice of $(x_0, \epsilon)$ such that for any $2 \le k \le m-1$, and for any $i\in [d]$
\[
\min_{\alpha^0\preceq x\preceq \beta^0}T_L^{1:m}(x)_i
- \min_{\alpha^{k}\preceq x\preceq\beta^k} T_L^{k:m}(x)_i
\le -B
\]
\end{restatable}
Proposition~\ref{prop:unrolling-hurts} shows that there exists network configuration where approximation of $T_L^{1:m}(x)$ (CROWN-IBP) can be made arbitrarily worse than that of $T_L^{k:m}(x)$ (IBP-Lin) for $2 \le k \le m-1$. 
This suggests that in addition to computational savings, using less layers of linear approximation may at times improve approximation tightness.
\bibliography{library}
\bibliographystyle{plain}
\clearpage
\appendix

\thispagestyle{empty}
\section{Proof of Proposition~\ref{prop:correctness}}
\begin{proposition}[Correctness of multi-layer linear approximation operator]
\label{prop:correctness}
Let $f^{\ell-k:\ell}$ be a sub-network of any $m$ layer feed-forward network, for some $\ell\in [m]$ and $k\in \{0,\dots,\ell-1\}$. Let $\tilde{f}^{\ell-k:\ell}$ be the last-layer modified version of $f^{\ell-k:\ell}$.
Suppose for any $k \in \{0, \dots, \ell-1\}$, there exists two vectors $(\hat\alpha^{\ell-k}, \hat\beta^{\ell-k})$ satisfying for any $\alpha^{0} \preceq x \preceq \beta^{0}$, 
\[
\hat\alpha^{\ell-k} 
\preceq \tilde{f}^{1:\ell-k}(x) 
\preceq \hat\beta^{\ell-k}
\mbox{~~~($\hat\alpha^{\ell-k}, \hat\beta^{\ell-k}$ as defined in Definition~\ref{defn:single-layer-relax})}
\]
Suppose there exists single-layer linear outer approximation of the network, where the layer-wise approximation parameters $D_L^{\ell-k}, D_U^{\ell-k}, b_L^{\ell-k}, b_U^{\ell-k}$ are constructed based on $(\hat\alpha^{\ell-k}, \hat\beta^{\ell-k})$.
Let $T_L^{\ell-k:\ell}$ and $T_U^{\ell-k:\ell}$ be a $k$-layer linear approximation to $\tilde{f}^{\ell-k:\ell}$ as defined in Definition~\ref{defn:multi-layer-lin-appx}. Then
for any $k \in \{0,  \dots, \ell-1\}$,
\[
T_L^{\ell-k:\ell}(f^{1: \ell-k-1}(x))\preceq \tilde{f}^{1: \ell} (x) \preceq T_U^{\ell-k:\ell}(f^{1:\ell-k-1}(x)),
~~\forall x \text{~s.t.~} \alpha^0\preceq x \preceq \beta^0
\]
(we abuse notation by letting $f^{1:0}(x):=x$).
In particular, $\forall x, \alpha^0\preceq x \preceq \beta^0$
\[
T_L^{1:m}(x)\preceq \tilde{f}^{1:m} (x) \preceq T_U^{1:m}(x)
\]
\end{proposition}
\begin{proof}
We prove the first statement by induction on $k$:
\paragraph{Base case:} For $k=0$,
\begin{eqnarray*}
T_L^{\ell:\ell} (f^{1:\ell-1}(x))
=
W_L^{\ell:\ell} f^{1:\ell-1} (x) + b_L^{\ell:\ell}
=
W^{\ell} f^{1:\ell-1} (x)  + b^{\ell}  
\mbox{~~(Definition~\ref{defn:multi-layer-lin-appx}}) \\
=
\tilde{f}^{1:\ell}(x)
~~~(\mbox{and hence~} \preceq \tilde{f}^{1:\ell}(x))
\end{eqnarray*}
The case for $T_U^{\ell:\ell}$ can be proved similarly. This shows the statement holds for $k=0$.
\paragraph{From $k$ to $k+1$:}
Suppose the statement holds for any $k \in \{0, \dots, \ell-1\}$, we next show it must hold for $k+1$. We expand $T_L^{\ell-k:\ell}(f^{1:\ell-k-1}(x))$ to reveal the term $f^{1:\ell-k-2}$:
\begin{eqnarray*}
T_L^{\ell-k:\ell}(f^{1:\ell-k-1}(x))
=
W_L^{\ell-k:\ell} f^{1:\ell-k-1}(x) + b_L^{\ell-k:\ell} \\
=
W_L^{\ell-k:\ell} g^{\ell-k-1} (W^{\ell-k-1} f^{1:\ell-k-2}(x) + b^{\ell-k-1}) + b_L^{\ell-k:\ell} \\
=
W_L^{\ell-k:\ell} g^{\ell-(k+1)} (W^{\ell-(k+1)} f^{1:\ell-(k+1)-1}(x) + b^{\ell-(k+1)}) + b_L^{\ell-k:\ell} 
\end{eqnarray*}
By property of pre-activation bounds $\hat\alpha^{\ell-(k+1)}$ and $\hat\beta^{\ell-(k+1)}$, for any $\alpha^0 \preceq x \preceq \beta^0$, 
\[
\hat\alpha^{\ell-(k+1)} \preceq 
\tilde{f}^{1:\ell-(k+1)}(x)
=W^{\ell-(k+1)} f^{1:\ell-(k+1)-1}(x) + b^{\ell-(k+1)} 
\preceq \hat\beta^{\ell-(k+1)}
\]
By property of single-layer linear approximation,  $\forall \hat z, ~s.t.~ \hat\alpha^{\ell-(k+1)} \preceq \hat z \preceq \hat\beta^{\ell-(k+1)}$
\[
D_L^{\ell-(k+1)} (\hat z) + b_L^{\ell-(k+1)}
\preceq
g^{\ell-(k+1)}(\hat z)
\]
It follows that for any $\alpha^0 \preceq x \preceq \beta^0$, 
\begin{eqnarray*}
D_L^{\ell-(k+1)} (W^{\ell-(k+1)} f^{1:\ell-(k+1)-1}(x) + b^{\ell-(k+1)} ) + b_L^{\ell-(k+1)} \\
=
D_L^{\ell-(k+1)} \tilde{f}^{1:\ell-(k+1)}(x) + b_L^{\ell-(k+1)} \\
\preceq
g^{\ell-(k+1)}(\tilde{f}^{1:\ell-(k+1)}(x)) \\
\end{eqnarray*}
Similarly, we get
\begin{eqnarray*}
g^{\ell-(k+1)}(\tilde{f}^{1:\ell-(k+1)}(x)) 
\preceq
D_U^{\ell-(k+1)} (\tilde{f}^{1:\ell-(k+1)}(x)) + b_U^{\ell-(k+1)} 
\end{eqnarray*}
Then for any $\alpha^0 \preceq x \preceq \beta^0$, for any $i \in [d_{\ell}]$, 
\begin{eqnarray*}
T_L^{\ell-k:\ell}(f^{1:\ell-k-1}(x))_i
=
W_{L, i}^{\ell-k:\ell} g^{\ell-(k+1)} (\tilde{f}^{1:\ell-(k+1)}(x)) + b_{L, i}^{\ell-k:\ell} \\
\ge
\min_{D_L^{\ell-(k+1)}\tilde{f}^{1:\ell-(k+1)}(x) + b_L^{\ell-(k+1)} \preceq y \preceq  D_U^{\ell-(k+1)}\tilde{f}^{1:\ell-(k+1)}(x) + b_U^{\ell-(k+1)}}
W_{L, i}^{\ell-k:\ell} y + b_{L, i}^{\ell-k:\ell}
\end{eqnarray*}
For a fixed $x$, the RHS is a minimization problem that can be solved exactly by applying Eq.~\eqref{eqn:ibp_lo}:
\begin{eqnarray*}
\min_{D_L^{\ell-(k+1)}\tilde{f}^{1:\ell-(k+1)}(x) + b_L^{\ell-(k+1)} \preceq y \preceq  D_U^{\ell-(k+1)}\tilde{f}^{1:\ell-(k+1)}(x) + b_U^{\ell-(k+1)}}
W_{L, i}^{\ell-k:\ell} y + b_{L, i}^{\ell-k:\ell} \\
=
W_{L, i}^{\ell-k:\ell} (D_U^{\ell-(k+1)}\tilde{f}^{1:\ell-(k+1)}(x) + b_U^{\ell-(k+1)}) \\
+ [W_{L, i}^{\ell-k:\ell}]_{+} (D_L^{\ell-(k+1)}\tilde{f}^{1:\ell-(k+1)}(x) \\
+ b_L^{\ell-(k+1)} - D_U^{\ell-(k+1)}\tilde{f}^{1:\ell-(k+1)}(x) - b_U^{\ell-(k+1)} ) + b_{L, i}^{\ell-k:\ell}\\
=
\bigg (W_{L, i}^{\ell-k:\ell}D_U^{\ell-(k+1)} + [W_{L, i}^{\ell-k:\ell}]_{+}(D_L^{\ell-(k+1)}- D_U^{\ell-(k+1)})  \bigg)\tilde{f}^{1:\ell-(k+1)}(x) \\
+ W_{L, i}^{\ell-k:\ell} b_U^{\ell-(k+1)} + [W_{L, i}^{\ell-k:\ell}]_{+} (b_L^{\ell-(k+1)} - b_U^{\ell-(k+1)})
+ b_{L, i}^{\ell-k:\ell} \\
=
P_{L, i}^{\ell-k} \tilde{f}^{1:\ell-(k+1)}(x) + [W_{L, i}^{\ell-k:\ell}]_{-} b_U^{\ell-(k+1)} + [W_{L, i}^{\ell-k:\ell}]_{+} b_L^{\ell-(k+1)} 
+ b_{L, i}^{\ell-k:\ell} \\
=
P_{L, i}^{\ell-k} (W^{\ell-(k+1)} f^{1:\ell-(k+1)-1}(x) + b^{\ell-(k+1)}) \\
+
[W_{L, i}^{\ell-k:\ell}]_{-} b_U^{\ell-(k+1)} + [W_{L, i}^{\ell-k:\ell}]_{+} b_L^{\ell-(k+1)} 
+ b_{L, i}^{\ell-k:\ell} \\
=
P_{L, i}^{\ell-k} W^{\ell-(k+1)} f^{1:\ell-(k+1)-1}(x) \\
+
P_{L, i}^{\ell-k} b^{\ell-(k+1)} + [W_{L, i}^{\ell-k:\ell}]_{-} b_U^{\ell-(k+1)} + [W_{L, i}^{\ell-k:\ell}]_{+} b_L^{\ell-(k+1)} 
+ b_{L, i}^{\ell-k:\ell} \\
=
W_{L, i}^{\ell-(k+1):\ell} f^{1:\ell-(k+1)-1}(x) + b_{L, i}^{\ell-(k+1):\ell} \\
=
T_{L}^{\ell-(k+1):\ell} (f^{1:\ell-(k+1)-1}(x))_i
\end{eqnarray*}
This shows that
\[
T_{L}^{\ell-k:\ell} (f^{1:\ell-k-1}(x))_i
\ge
T_{L}^{\ell-k-1:\ell} (f^{1:\ell-k-2}(x))_i
\]
By inductive hypothesis,
\[
f^{1:\ell} (x)_i \ge T_{L}^{\ell-k:\ell} (f^{1:\ell-k-1}(x))_i
\ge
T_{L}^{\ell-k-1:\ell} (f^{1:\ell-k-2}(x))_i
\]
The case for $T_{U}^{\ell-k-1:\ell} (f^{1:\ell-k-2}(x))$ can be proved similarly.
Thus, it can be shown inductively that first statement holds for any $k \in \{0, \dots \ell-1\}$. 
The second statement can be shown by letting $\ell=m$ and $k=\ell-1$ and noting that $f^{1:0}(x)=x$.
\end{proof}
\section{Proof of Corollary~\ref{coro:correctness}}
\correctness*
\begin{proof}
To prove the first statement,
we first claim that for any $\ell \in [m]$, it holds that $\forall x ~s.t.~\alpha^{0} \preceq x \preceq \beta^{0}$,
\[
\hat\alpha^{\ell} \preceq \tilde{f}^{1:\ell}(x) = W^{\ell} f^{1:\ell-1}(x) + b^{\ell} \preceq \hat\beta^{\ell}
\]
Since the algorithm's instruction implies that $(\hat\alpha^{\ell}, \hat\beta^{\ell})$ is either calculated by IBP iteration or is an intersection of previously computed pre-activation bounds and the LinApprox pre-activation bounds, we divide the proof of the claim by these two cases.
\paragraph{Case 1: IBP pre-activation bounds}
In this case, Lemma~\ref{lm:correctness_ibp} shows the first statement holds
\paragraph{Case 2: An intersection of existing and LinApprox bounds}
In this case, we can prove the first statement by induction on the number of times the existing bound gets updated: If it's the first time, it corresponds to the case where the existing bound is IBP pre-activation bounds, which satisfies the condition for Proposition~\ref{prop:correctness} to hold by Lemma~\ref{lm:correctness_ibp}. So to prove the statement holds for LinApprox bounds, we can apply Proposition~\ref{prop:correctness} on the sub-network $\tilde{f}^{1:\ell}$: First, note that by Eq~\eqref{eqn:ibp_lo}, $\hat\alpha^{\prime}$ (line 7 of Algorithm \ref{algo:lin-appx}) satisfies for any $i\in [d_{\ell}]$,
\begin{eqnarray*}
\hat\alpha_i^{\prime}
=
(W_L^{1:\ell})_i\beta^{0} 
   + [(W_L^{1:\ell})_i]_{+} (\alpha^{0}-\beta^{0}) + (b_L^{1:\ell})_i \\
=
\min_{\alpha^{0}\preceq x\preceq \beta^{0}}
(W_L^{1:\ell})_i x + (b_L^{1:\ell})_i 
=
\min_{\alpha^{0}\preceq x\preceq \beta^{0}}
(T_L^{1:\ell}(x))_i
\end{eqnarray*}
By Proposition~\ref{prop:correctness}, 
\[
\hat\alpha_i^{\prime}
=
\min_{\alpha^{0}\preceq x\preceq \beta^{0}}
(T_L^{1:\ell}(x))_i
\le
\min_{\alpha^{0}\preceq x\preceq \beta^{0}}\tilde{f}^{1:\ell}_i(x)
\]
and similarly by Eq~\eqref{eqn:ibp_hi} and Proposition~\ref{prop:correctness},
\[
\hat\beta_i^{\prime}
=\max_{\alpha^0\preceq x\preceq \beta^0} T_U^{1:\ell}(x)_i 
\ge 
\max_{\alpha^0\preceq x\preceq \beta^0} \tilde{f}^{1:\ell}(x)_i
\]
This shows the first statement holds for the intersection of IBP pre-activation bounds and the LinApprox bounds based off it.
Now, using the same argument we can prove inductively the first statement holds: That is, we can assume the first statement holds when the existing bound is updated for the $k$-th time (replacing the IBP pre-activation bound with the updated $\hat\alpha, \hat\beta$) and show that it holds when the existing bound is updated for $k+1$-th time.

We prove the second statement by noting that
\[
\alpha^m = g^m(\hat\alpha^m)=\hat\alpha^m
\preceq \tilde{f}^{1:m}(x)
\]
where the inequality is implied by applying the first statement to $\tilde{f}^{1:m}$. This concludes the proof.
\end{proof}
\begin{lemma}[Correctness of IBP pre-activation bounds]
\label{lm:correctness_ibp}
Let $f^{1:m}$ be a $m$-layer feed-forward network, and let $(\hat\alpha^{\ell}, \hat\beta^{\ell}), \ell\in [m]$ be the pre-activation bounds calculated by IBP-Lin with null schedule vector (i.e., IBP). Then for any $x$ such that $\alpha^0\preceq x \preceq \beta^0$,
\[
\hat\alpha^{\ell}\preceq \tilde{f}^{1:\ell}(x) \preceq \hat\beta^{\ell}
\]
\end{lemma}
\begin{proof}
We prove the lemma by induction on $\ell$:
\paragraph{Base case $\ell=1$:}
By definition of pre-activation IBP bound, for any $i \in [d_{1}]$,
\begin{eqnarray*}
\hat\alpha_i^{1}
=
\min_{\alpha^{0} \preceq x \preceq \beta^{0}} W_i^{1} x + b_i^{1}
= \min_{\alpha^0\preceq x \preceq \beta^0}\tilde{f}_i^{1:1}(x)
\le
\tilde{f}_i^{1:1}(x), ~~~\forall \alpha^0\preceq x \preceq \beta^0
\end{eqnarray*}
and
\begin{eqnarray*}
\hat\beta_i^{1}
=
\max_{\alpha^{0} \preceq x \preceq \beta^{0}} W_i^{1} x + b_i^{1}
= \max_{\alpha^0\preceq x \preceq \beta^0}\tilde{f}_i^{1:1}(x)
\ge
\tilde{f}_i^{1:1}(x), ~~~\forall \alpha^0\preceq x \preceq \beta^0
\end{eqnarray*}
Therefore,
\[
\hat\alpha^{1}\preceq \tilde{f}^{1:1}(x) \preceq \hat\beta^{1}
~~~\forall \alpha^0\preceq x \preceq \beta^0
\]
This shows the claim holds for the base case.
\paragraph{From $\ell$ to $\ell+1$:}
Suppose the statement holds for some $\ell$, we show it holds for $\ell+1$.
By definition, for any $i \in [d_{\ell+1}]$,
\[
\hat\alpha_i^{\ell+1} = \min_{\alpha^{\ell}\preceq x\preceq \beta^{\ell}} W_i^{\ell+1} x + b_i^{\ell+1}
\]
where
$
\alpha^{\ell} = g^{\ell} (\hat\alpha^{\ell})
$ and for any $i\in d_{\ell}$,
$
\alpha_i^{\ell} = g^{\ell} (\hat\alpha_i^{\ell}) 
$.
By the inductive hypothesis 
\[
\hat\alpha_i^{\ell}\le \tilde{f}_i^{1:\ell}(x),~~\forall i\in[d_{\ell}]
\]
and since $g^{\ell}$ is non-decreasing,
\[
\alpha_i^{\ell} = g^{\ell} (\hat\alpha_i^{\ell}) \le g^{\ell} (\tilde{f}_i^{1:\ell}(x))
= f_i^{1:\ell} (x)
\]
Similarly, it can be derived that for any $\alpha^0\preceq x \preceq \beta^0$,
\[
f^{1:\ell}(x) \preceq \beta^{\ell} 
\]
And thus, for any $x$ such that $\alpha^{0} \preceq x \preceq \beta^{0}$, for any $i\in [d_{\ell+1}]$, 
\begin{eqnarray*}
\hat\alpha_i^{\ell+1} =\min_{\alpha^{\ell}\preceq x\preceq \beta^{\ell}} W_i^{\ell+1} x + b_i^{\ell+1}
\le W_i^{\ell+1} f^{1:\ell}(x) + b_i^{\ell+1}, \mbox{~~for any~} \alpha^{0} \preceq x \preceq \beta^{0}
\end{eqnarray*}
which implies
\[
\hat\alpha^{\ell+1}\preceq W^{\ell+1} f^{1:\ell}(x) + b^{\ell+1}
=
\tilde{f}^{1:\ell+1}(x)
\mbox{~~for any~} \alpha^{0} \preceq x \preceq \beta^{0}
\]
Similarly, it can be derived that
\[
\tilde{f}^{1:\ell+1}(x)
\preceq 
\hat\beta^{\ell+1}
\mbox{~~for any~} \alpha^{0} \preceq x \preceq \beta^{0}
\]
which proves the lemma by induction.
\end{proof}
%
%
\section{Proof of Proposition~\ref{thm:parallel-relax-error}}
\parallel*
\begin{proof}
We first prove the case for the lower bound on $\min_{\alpha^{\ell-k-1} \preceq x \preceq \beta^{\ell-k-1}}T_L^{\ell-k:\ell}(x)_i$.
Let $x^{\star}$ denote a minimizer of the modified network function at dimension $i$, i.e.,
$$
\tilde{f}_i^{\ell-k:\ell} (x^{\star})
=
\min_{\alpha^{\ell-k-1} \preceq x \preceq \beta^{\ell-k-1}} \tilde{f}_i^{\ell-k:\ell} (x)
$$
and let $\hat{x}$ denote a minimizer of the lower bound function at dimension $i$, 
$$
T_L^{\ell-k:\ell} (\hat{x})_i
=
\min_{\alpha^{\ell-k-1} \preceq x \preceq \beta^{\ell-k-1}} T_L^{\ell-k:\ell} (x)_i
$$
Since 
\begin{eqnarray*}
\tilde{f}_i^{\ell-k:\ell} (x^{\star}) - T_L^{\ell-k:\ell} (\hat{x})_i 
=
\underbrace{\tilde{f}_i^{\ell-k:\ell} (x^{\star}) - \tilde{f}_i^{\ell-k:\ell} (\hat{x})}_{A}
+ 
\underbrace{\tilde{f}_i^{\ell-k:\ell} (\hat{x})- T_L^{\ell-k:\ell}(\hat{x})_i}_{B} 
\end{eqnarray*}
Term $A \le 0$ by minimality of $\tilde{f}_i^{\ell-k:\ell} (x^{\star})$, and
we can bound the difference $B$ by the fact that $\tilde{f}_i^{\ell-k:\ell}(x)$ is between $T_L^{\ell-k:\ell}(x)$ and $T_U^{\ell-k:\ell}(x)$ for any $x$ within the constraint set:
\begin{eqnarray*}
B
\le
\max_{\alpha^{\ell-k-1} \preceq x \preceq \beta^{\ell-k-1}} ~\tilde{f}_i^{\ell-k:\ell} (x) - T_L^{\ell-k:\ell}(x)_i \\
\le
\max_{\alpha^{\ell-k-1} \preceq x \preceq \beta^{\ell-k-1}} ~T_U^{\ell-k:\ell}(x)_i - T_L^{\ell-k:\ell}(x)_i
\end{eqnarray*}
So we can upper bound $\tilde{f}_i^{\ell-k:\ell} (\hat{x}) - T_L^{\ell-k:\ell} (\hat{x})_i$ by upper bounding $\max_{x} ~T_U^{\ell-k:\ell}(x)_i - T_L^{\ell-k:\ell}(x)_i$, where
\begin{eqnarray*}
T_U^{\ell-k:\ell}(x)_i - T_L^{\ell-k:\ell}(x)_i
= (W_U^{\ell-k:\ell}-W_L^{\ell-k:\ell})_i x + (b_U^{\ell-k:\ell} - b_L^{\ell-k:\ell})_i
\end{eqnarray*}
We first prove by induction on $k$ that $W^{\ell-k:\ell}_U = W^{\ell-k:\ell}_L$ for parallelogram approximation.
For $k=0$, 
$
W_U^{\ell:\ell}-W_L^{\ell:\ell} = W^{\ell}-W^{\ell} = 0 \,. 
$
For any $k$, given that $W_U^{\ell-k:\ell}-W_L^{\ell-k:\ell} = 0$, we have
\begin{eqnarray*}
W_U^{\ell-k-1:\ell}-W_L^{\ell-k-1:\ell}
=
(P^{\ell-k}_U - P^{\ell-k}_L)W^{\ell-k-1} \\
=
(W^{\ell-k:\ell}_U - W^{\ell-k:\ell}_L)D_U^{\ell-k-1}W^{\ell-k-1} = 0
\end{eqnarray*}
where in the derivation above the terms involving $D_L^{\ell-k}-D_U^{\ell-k}$ equal to zero because we are applying parallelogram approximation here. By induction, we have
\[
W_U^{\ell-k:\ell}-W_L^{\ell-k:\ell} = 0
\]
Let $W^{\ell-k:\ell}:=W_U^{\ell-k:\ell}=W_L^{\ell-k:\ell}$. We get
\begin{eqnarray*}
T_U^{\ell-k:\ell}(x) - T_L^{\ell-k:\ell}(x)
=
b_U^{\ell-k:\ell} - b_L^{\ell-k:\ell} 
=
(P_U^{\ell-k+1:\ell} - P_L^{\ell-k+1:\ell})b^{\ell-k}  \\
+
([W_U^{\ell-k+1:\ell}]_{+} - [W_U^{\ell-k+1:\ell}]_{-})b^{\ell-k}_U 
+
([W_U^{\ell-k+1:\ell}]_{-} - [W_U^{\ell-k+1:\ell}]_{+})b^{\ell-k}_L \\
+
b_U^{\ell-k+1:\ell} - b_L^{\ell-k+1:\ell} \\
=
(W^{\ell-k+1:\ell}_U - W^{\ell-k+1:\ell}_L)D_U^{\ell-k} b^{\ell-k}\\ 
+
|W_U^{\ell-k+1:\ell}| b^{\ell-k}_U - |W_U^{\ell-k+1:\ell}| b^{\ell-k}_L  + b_U^{\ell-k+1:\ell} - b_L^{\ell-k+1:\ell} \\
=
|W^{\ell-k+1:\ell}| (b^{\ell-k}_U - b^{\ell-k}_L) + b_U^{\ell-k+1:\ell} - b_L^{\ell-k+1:\ell} \\
=
\dots
= \\
|W^{\ell-k+1:\ell}| (b^{\ell-k}_U - b^{\ell-k}_L)
+
|W^{\ell-k+2:\ell}| (b^{\ell-k+1}_U - b^{\ell-k+1}_L)
+ \dots +
|W^{\ell:\ell}| (b^{\ell-1}_U - b^{\ell-1}_L) \\
=
\sum_{s=0}^{k-1} |W^{\ell-s:\ell}| (b^{\ell-s-1}_U - b^{\ell-s-1}_L)
\end{eqnarray*}
So for each dimension $i$,
\[
T_U^{\ell-k:\ell}(x)_i - T_L^{\ell-k:\ell}(x)_i
=
\sum_{s=0}^{k-1} |W_i^{\ell-s:\ell}| (b^{\ell-s-1}_U - b^{\ell-s-1}_L)
\]
Note that this quantity does not depend on $x$, so
\begin{eqnarray*}
\tilde{f}_i^{\ell-k:\ell} (x^{\star}) - T_L^{\ell-k:\ell} (\hat{x})_i
\le
\max_x T_U^{\ell-k:\ell}(x)_i - T_L^{\ell-k:\ell}(x)_i \\
=
\sum_{s=0}^{k-1} |W_i^{\ell-s:\ell}| (b^{\ell-s-1}_U - b^{\ell-s-1}_L)
\end{eqnarray*}
For the inequality involving 
\[
\max_{\alpha^{\ell-k-1} \preceq x \preceq \beta^{\ell-k-1}}T_U^{\ell-k:\ell}(x)_i 
-
\max_{\alpha^{\ell-k-1}\preceq x \preceq \beta^{\ell-k-1}} \tilde{f}_i^{\ell-k:\ell}(x) \, ,
\] 
We similarly have
\begin{eqnarray*}
\max_{\alpha^{\ell-k-1} \preceq x \preceq \beta^{\ell-k-1}}T_U^{\ell-k:\ell}(x)_i - \max_{\alpha^{\ell-k-1}\preceq x \preceq \beta^{\ell-k-1}} \tilde{f}_i^{\ell-k:\ell}(x) \\
\le
\max_{\alpha^{\ell-k-1} \preceq x \preceq \beta^{\ell-k-1}} ~ T_U^{\ell-k:\ell}(x)_i - T_L^{\ell-k:\ell}(x)_i
\end{eqnarray*}
So we get the exact same bound as before.
\end{proof}
\begin{lemma}[Property of parallelogram relaxation]
\label{lemma:paralellogram}
Let $f^{1:m}$ be a neural network that can be layer-wise outer approximated by parallelogram relaxation. For any $\ell$ and $0<k<\ell$, let $\tilde{f}^{\ell-k:\ell}$ denote the last-layer modified version of sub-network $f^{\ell-k:\ell}$. Let $T_L^{\ell-k:\ell}, T_U^{\ell-k:\ell}$ be a multi-layer outer approximation of $\tilde{f}^{\ell-k:\ell}$ based on parallelogram relaxation.
Then
1). $W_L^{\ell-k:\ell} = W_U^{\ell-k:\ell}$
and
2). $T_U^{\ell-k:\ell}(x) - T_L^{\ell-k:\ell}(x)$ doesn't depend on $x$.
\end{lemma}
\begin{proof}
We prove by induction on $k$.
\begin{eqnarray*}
T_U^{\ell-k:\ell}(x)_i - T_L^{\ell-k:\ell}(x)_i
= (W_U^{\ell-k:\ell}-W_L^{\ell-k:\ell})_i x + (b_U^{\ell-k:\ell} - b_L^{\ell-k:\ell})_i
\end{eqnarray*}
For $k=0$, by property of parallelogram approximation,
\[
W_U^{\ell:\ell}-W_L^{\ell:\ell} = W^{\ell}-W^{\ell} = 0 
\mbox{~~and~~}
b_U^{\ell:\ell} - b_L^{\ell:\ell} = b^{\ell} - b^{\ell} = 0
\]
For any $k$, given that $W_U^{\ell-k:\ell}-W_L^{\ell-k:\ell} = 0$, we have
\begin{eqnarray*}
W_U^{\ell-k-1:\ell}-W_L^{\ell-k-1:\ell}
=
(P^{\ell-k}_U - P^{\ell-k}_L)W^{\ell-k-1} \\
=
(W^{\ell-k:\ell}_U - W^{\ell-k:\ell}_L)D_U^{\ell-k-1}W^{\ell-k-1} = 0
\end{eqnarray*}
where in the derivation above the terms involving $D_L^{\ell-k}-D_U^{\ell-k}$ equal to zero due to the property of parallelogram relaxation again. By induction, we have
\[
W_U^{\ell-k:\ell}-W_L^{\ell-k:\ell} = 0
\]
This implies that for any $x$ and $\forall i\in [d_{\ell}]$
\[
T_U^{\ell-k:\ell}(x)_i - T_L^{\ell-k:\ell}(x)_i = (b_U^{\ell-k:\ell} - b_L^{\ell-k:\ell})_i
\]
which is independent of $x$.
\end{proof}
%
%
\section{Proof of Proposition~\ref{prop:single_layer_lower_bound}}
\singlelayer*
\begin{proof}
To prove the first statement, any neuron $i$ satisfies
\[
\min_x D_i^1W^1x+(b_L^1)_i
=
\min_x D_i^1W^1x
=
D_{ii}^1 \min_x W_i^1x
=
D_{ii}^1\hat{\alpha}_i^1
\]
For $i\in\mathcal{I}$, $D_{ii} = \frac{\hat{\beta}_i^1}{\hat{\beta}_i^1 - \hat{\alpha}_i^1}$ and
\[
\min_x D_i^1W^1x+(b_L^1)_i 
=
\frac{\hat{\beta}_i^1}{\hat{\beta}_i^1 - \hat{\alpha}_i^1} \hat{\alpha}_i^1
\]
where
\[
\hat\alpha^1_i = W_i^1x_0 -|W_i^1|\epsilon\vec{1} = W_i^1x_0-\epsilon \|W_i^1\|_1
\mbox{
~~
and
~~
}
\hat\beta^1_i = W_i^1x_0 + |W_i^1|\epsilon\vec{1} = W_i^1x_0 + \epsilon \|W_i^1\|_1
\]
The expected gap is
\begin{eqnarray*}
\E \big( \min_x D_i^1W^1x+(b_L^1)_i - \min_x [W_i^1x]_{+} \big) \\
=
\p(i\in\mathcal{I}) \E\big(\min_x D_i^1W^1x+(b_L^1)_i - \min_x [W_i^1x]_{+} | i\in\mathcal{I}\big) \\
+
(1-\p(i\in\mathcal{I}))\E \big(\min_x D_i^1W^1x+(b_L^1)_i - \min_x [W_i^1x]_{+} | i\notin\mathcal{I}\big) \\
\le 
\p(i\in\mathcal{I}) \E\big(\min_x D_i^1W^1x+(b_L^1)_i - \min_x [W_i^1x]_{+} | i\in\mathcal{I}\big) \\
\le 
\p(i\in\mathcal{I}) \E\big(\min_x D_i^1W^1x+(b_L^1)_i | i\in\mathcal{I}\big) \\
=
\p(i\in\mathcal{I}) \E\big(\frac{\hat{\beta}_i^1}{\hat{\beta}_i^1 - \hat{\alpha}_i^1} \hat{\alpha}_i^1 | i\in\mathcal{I}\big) \\
=
\p(i\in\mathcal{I}) \E\big(\frac{(W_i^1x_0)^2 - (\epsilon \|W_i^1\|_1)^2}{2\epsilon \|W_i^1\|_1} \big | i\in\mathcal{I}\big)
\end{eqnarray*}
where the first inequality is due to the correctness of IBP-Lin, which implies
$
\min_x D_i^1W^1x+(b_L^1)_i - \min_x [W_i^1]_{+} \le 0
$
always hold, and the second inequality is due to non-negativity of $\min_x [W_i^1x]_{+}$.
Since 
\[
\{i\in\mathcal{I}\} 
=
\{W_i^1x_0 < \epsilon\|W_i^1\|_1\}\cap \{W_i^1x_0 > -\epsilon\|W_i^1\|_1\}
=
\{(W_i^1x_0)^2 < (\epsilon\|W_i^1\|_1)^2\}
\]
We get 
\begin{eqnarray*}
\E\big(\frac{(W_i^1x_0)^2 - (\epsilon \|W_i^1\|_1)^2}{2\epsilon \|W_i^1\|_1} \big | i\in\mathcal{I}\big)
=
\frac{\E \big( (W_i^1x_0)^2 | i\in\mathcal{I}\big) - (\epsilon \|W_i^1\|_1)^2 }{2\epsilon \|W_i^1\|_1} \\ 
\le 
\frac{\E (W_i^1x_0)^2 - (\epsilon \|W_i^1\|_1)^2 }{2\epsilon \|W_i^1\|_1}
=
\frac{\|W_i^1\|_2^2 - (\epsilon \|W_i^1\|_1)^2 }{2\epsilon \|W_i^1\|_1}
\end{eqnarray*}
Substituting $\kappa_i = \frac{\|W_i^1\|_1}{\|W_i^1\|_2}$ in the inequality above, we get
\begin{eqnarray*}
\E\big(\frac{(W_i^1x_0)^2 - (\epsilon \|W_i^1\|_1)^2}{2\epsilon \|W_i^1\|_1} \big | i\in\mathcal{I}\big)
\le
\frac{\|W_i^1\|_2^2 - (\epsilon \kappa_i \|W_i^1\|_2)^2}{2 \epsilon \kappa_i \|W_i^1\|_2} \\
=
\frac{\|W_i^1\|_2 - (\epsilon \kappa_i)^2 \|W_i^1\|_2}{2\epsilon \kappa_i}
=
\frac{\|W_i^1\|_2}{2} (\frac{1}{\epsilon\kappa_i} - \epsilon\kappa_i) 
\end{eqnarray*}
By property of Gaussian distribution, we can explicitly represent $\p(i\in\mathcal{I})$ as
\[
\p (|W_i^1x_0| < \epsilon \|W_i^1\|_1)
=
\p (|\frac{W_i^1x_0}{\|W_i^1\|_2}| < \epsilon \frac{\|W_i^1\|_1}{\|W_i^1\|_2})
=
2\Phi(\epsilon \frac{\|W_i^1\|_1}{\|W_i^1\|_2})-1 \ge 0
\]
Here the second equality holds because $\frac{W_i^1x_0}{\|W_i^1\|_2}$ has standard normal distribution.
Therefore, 
\begin{eqnarray*}
\E \big( \min_x D_i^1W^1x+(b_L^1)_i - \min_x [W_i^1x]_{+} \big)  \\
\le 
\p(i\in\mathcal{I}) \E\big(\frac{(W_i^1x_0)^2 - (\epsilon \|W_i^1\|_1)^2}{2\epsilon \|W_i^1\|_1} \big | i\in\mathcal{I}\big) \\ 
\le 
(2\Phi(\epsilon \kappa_i)-1)\frac{\|W_i^1\|_2}{2} (\frac{1}{\epsilon\kappa_i} - \epsilon\kappa_i) 
=
(\Phi(\epsilon \kappa_i)-1/2)\|W_i^1\|_2 (\frac{1}{\epsilon\kappa_i} - \epsilon\kappa_i) 
\end{eqnarray*}
To prove the second statement, consider $D_i^1W^1x+(b_U^1)_i$ for any $i$. For $i\in \mathcal{I}_{-}\cup \mathcal{I}_{+}$, 
\[
D_i^1W^1x+(b_U^1)_i = D_i^1W^1x = [W_i^1x]_{+}
\]
For $i\in\mathcal{I}$,
\[
\max_x D_i^1W^1 x + (b_U^1)_i
=
D_i^1\max_x W^1 x + (b_U^1)_i
=
\frac{\hat\beta_i^1}{\hat\beta_i^1 - \hat\alpha_i^1} \hat\beta_i^1
- \frac{\hat\alpha_i^1\hat\beta_i^1}{\hat\beta_i^1 - \hat\alpha_i^1}
=
\hat\beta_i^1
\]
And for $i\in\mathcal{I}$,
\[
\max_x [W_i^1 x]_{+} = [\hat\beta_i^1]_{+} = \hat\beta_i^1
\]
which finishes proving the second statement.
\end{proof}
\section{Proof of Proposition~\ref{prop:tight-bound}}
\tightbound*
\begin{proof}
By Proposition~\ref{thm:parallel-relax-error},
\begin{eqnarray*}
\min_{\alpha^{\ell-k-1} \preceq x \preceq \beta^{\ell-k-1}}T_L^{\ell-k:\ell}(x)_i \ge \min_{\alpha^{\ell-k-1} \preceq x \preceq \beta^{\ell-k-1}} \tilde{f}_i^{\ell-k:\ell}(x) \\
- 
\sum_{s=0}^{k-1} |W_i^{\ell-s:\ell}| (b^{\ell-s-1}_U - b^{\ell-s-1}_L)
\end{eqnarray*}
By Lemma~\ref{lemma:zero-width}, for any $s=\{0,\dots,\ell-3\}$,
\[
b_U^{\ell-s-1} = b_L^{\ell-s-1} 
\]
This implies that for any $1 \le k \le \ell-2$,
\[
\sum_{s=0}^{k-1} |W_i^{\ell-s:\ell}| (b^{\ell-s-1}_U - b^{\ell-s-1}_L)
=
0
\]
So
\[
\min_{\alpha^{\ell-k-1} \preceq x \preceq \beta^{\ell-k-1}}T_L^{\ell-k:\ell}(x)_i \ge \min_{\alpha^{\ell-k-1} \preceq x \preceq \beta^{\ell-k-1}} \tilde{f}_i^{\ell-k:\ell}(x)
\]
But by correctness of LinApprox (Proposition~\ref{prop:correctness} applied to the sub-network $\tilde{f}^{\ell-k:\ell}$),
\[
\min_{\alpha^{\ell-k-1} \preceq x \preceq \beta^{\ell-k-1}}T_L^{\ell-k:\ell}(x)_i \le \min_{\alpha^{\ell-k-1} \preceq x \preceq \beta^{\ell-k-1}} \tilde{f}_i^{\ell-k:\ell}(x)
\]
This proves the first statement. The second statement can be proved similarly.
\end{proof}
%
%
\section{Proof of Proposition~\ref{prop:transition}}
\transition*
\begin{proof}
\textbf{Proof of first statement:} By Lemma~\ref{lemma:zero-width}, for any $s \in \{0, \dots, m-3\}$, $b_U^{m-s-1} - b_L^{m-s-1} = 0$. This implies that for any $1\le k\le m-2$,
\[
\sum_{s=0}^{k-1} |W_i^{m-s:m}|(b_U^{m-s-1}-b_L^{m-s-1}) = 0
\]
By Proposition~\ref{thm:parallel-relax-error}, for any $k\in\{0,\dots,m-1\}$,
\[
\min_{\alpha^0\preceq x\preceq\beta^0}T_L^{m-k:m}(x)_i
-
\min_{\alpha^0\preceq x\preceq\beta^0}\tilde{f}_i^{m-k:m}(x)
\ge
-\sum_{s=0}^{k-1} |W_i^{m-s:m}|(b_U^{m-s-1}-b_L^{m-s-1})
\]
Let $k=m-1$, 
\begin{eqnarray*}
\min_{\alpha^0\preceq x\preceq\beta^0}T_L^{1:m}(x)_i
-
\min_{\alpha^0\preceq x\preceq\beta^0}\tilde{f}_i^{1:m}(x) 
\ge
-\sum_{s=0}^{m-2} |W_i^{m-s:m}|(b_U^{m-s-1}-b_L^{m-s-1}) \\
=
-\sum_{s=0}^{m-3} |W_i^{m-s:m}|(b_U^{m-s-1}-b_L^{m-s-1})
- |W_i^{2:m}|(b_U^{1}-b_L^{1}) \\
=
- |W_i^{2:m}|(b_U^{1}-b_L^{1})
\end{eqnarray*}
Temporarily let $r:=\frac{1}{d^p}$. By Lemma~\ref{lemma:approx-error-prop-from-first-layer},
\[
|W^{2:m}| = |r^{m-1} d^{m-2} \vec{1}\vec{1}^T| = r^{m-1} d^{m-2} \vec{1}\vec{1}^T
\]
Now consider $b_U^1-b_L^1=b_U^1$, which is determined by pre-activation vectors $\hat\alpha^1, \hat\beta^1$:
\[
\hat\alpha^1 = W^1\beta^0 + [W^1]_{+}(\alpha^0-\beta^0)
= W^1\alpha^0 = W^1x_0 - \epsilon W^1\vec{1}
\]
and
\[
\hat\beta^1 = W^1\beta^0 + [-W^1]_{+}(\beta^0-\alpha^0)
=  W^1\beta^0 = W^1x_0 + \epsilon W^1\vec{1}
\]
For any $i$ such that $i\in \mathcal{I}$, 
\begin{eqnarray*}
(b_U^1)_i 
= \frac{-\hat\alpha_i^1\hat\beta_i^1}{\hat\beta_i^1-\hat\alpha_i^1}
= \frac{-((W_i^1x_0)^2 - (\epsilon W_i^1 \vec{1})^2)}{2\epsilon W_i^1\vec{1}} \\
\le 
\frac{(\epsilon W_i^1 \vec{1})^2}{2\epsilon W_i^1\vec{1}}
=
\frac{\epsilon \|W_i^1\|_1}{2}
=
\frac{\epsilon r d}{2}
\end{eqnarray*}
For $i\notin \mathcal{I}$, $(b_U^1)_i=0 \le \frac{\epsilon rd}{2}$.
Therefore, we can conclude that
\begin{eqnarray*}
\min_{\alpha^0\preceq x\preceq\beta^0}T_L^{1:m}(x)_i
-
\min_{\alpha^0\preceq x\preceq\beta^0}\tilde{f}_i^{1:m}(x) \\
\ge
- |W_i^{2:m}|(b_U^{1}-b_L^{1})
=
- r^{m-1} d^{m-2} \vec{1}_i\vec{1}^T \frac{\epsilon r d}{2}\vec{1}
= -\frac{\epsilon}{2} r^m d^m 
= - \frac{\epsilon}{2} \frac{1}{d^{m(p-1)}}
\end{eqnarray*}
Since $m \ge \frac{1}{p-1}\log_d \frac{\epsilon}{2\delta}$, we get
\[
(\frac{1}{d^{p-1}})^m \frac{1}{2}\epsilon \le \delta
\]
And
\[
\min_{\alpha^0\preceq x\preceq\beta^0}T_L^{1:m}(x)_i
-
\min_{\alpha^0\preceq x\preceq\beta^0}\tilde{f}_i^{1:m}(x) \\
\ge
-\delta
\]
\paragraph{Proof of second statement:}
By the first statement of Lemma~\ref{lemma:approx-error-prop-from-first-layer},
$T_L^{1:m}(x)=W^{1:m}x$. Temporarily let denote $r=\frac{1}{d^p}$; by the second statement of Lemma~\ref{lemma:approx-error-prop-from-first-layer},
\begin{eqnarray*}
\min_{\alpha^0 \preceq x\preceq\beta^0} T_L^{1:m}(x)_i= \min_{\alpha^0 \preceq x\preceq\beta^0} W_i^{1:m}x
= \min_{\alpha^0 \preceq x\preceq\beta^0} W_i^{2:m}D^1W^1 x \\
= W_i^{2:m}D^1W^1 \beta^0 + [W_i^{2:m}D^1W^1]_{+} (\alpha^0-\beta^0)
= W_i^{2:m}D^1W^1\alpha^0
\end{eqnarray*}
On the other hand, since for all $\ell \in [m]$, $W_{ij}^{\ell} \ge 0$, 
\begin{eqnarray*}
\min_x \tilde{f}_i^{1:m}(x)
=
\min_x W_i^m\dots [ W^2 [W^1x]_{+}]_{+} \\
=
\min_x W_i^m\dots W^2 [W^1x]_{+}
=
\min_x W_i^{2:m}[W^1x]_{+} 
\end{eqnarray*}
Since by property of IBP activation bounds, for all $\alpha^0 \preceq x \preceq \beta^0$,
\[
[\hat{\alpha}^1]_{+}\preceq [W^1x]_{+} \preceq [\hat{\beta}^1]_{+}
\]
So
\begin{eqnarray*}
\min_{\alpha^0\preceq x\preceq\beta^0} \tilde{f}_i^{1:m}(x)
=
\min_{\alpha^0\preceq x\preceq\beta^0} W_i^{2:m}[W^1x]_{+} \\
\ge
\min_{[\hat{\alpha}^1]_{+}\preceq y\preceq [\hat{\beta}^1]_{+}} W_i^{2:m}y \\
=
W_i^{2:m}[\hat{\beta}^1]_{+} + [W_i^{2:m}]_{+} ([\hat{\alpha}^1]_{+} - [\hat{\beta}^1]_{+}) \\
=
W_i^{2:m} [\hat{\alpha}^1]_{+}
\end{eqnarray*}
Therefore,
\begin{eqnarray*}
\min_{\alpha^0 \preceq x\preceq\beta^0} T_L^{1:m}(x)_i
-
\min_{\alpha^0 \preceq x\preceq\beta^0} \tilde{f}_i^{1:m}(x) \\
\le
W_i^{2:m}D^1W^1\alpha^0
-
W_i^{2:m} [\hat{\alpha}^1]_{+}
\end{eqnarray*}
By definition of IBP pre-activation bounds, for all $i$,
\[
\hat{\alpha}_i^1
=
W_i^1 \beta^0 + [W_i^1]_{+} (\alpha^0 - \beta^0)
= 
W_i^1\alpha^0
\]
and thus,
\[
\min_{\alpha^0 \preceq x\preceq\beta^0} T_L^{1:m}(x)_i
-
\min_{\alpha^0 \preceq x\preceq\beta^0} \tilde{f}_i^{1:m}(x)
\le
W_i^{2:m}D^1\hat{\alpha}^1
-
W_i^{2:m} [\hat{\alpha}^1]_{+}
\]
Consider any $j$,
\[
D_j^1 \hat{\alpha}^1=
D_{jj}^1 \hat{\alpha}_j^1
=
\begin{cases}
\hat{\alpha}_j^1 = [\hat{\alpha}_j^1]_{+} ~~j\in\mathcal{I}_{+} \\
0 = [\hat{\alpha}_j^1]_{+} ~~j\in\mathcal{I}_{-} \\
\frac{\hat{\beta}_j^1}{\hat{\beta}_j^1 - \hat{\alpha}_j^1}\hat{\alpha}_j^1 ~~j\in\mathcal{I}
\end{cases}
\]
This implies that 
\[
D_j^1\hat{\alpha}^1 - [\hat{\alpha}_j^1]_{+}
=
\begin{cases}
0 \text{~~if~~} j\notin \mathcal{I} \\
\frac{\hat{\beta}_j^1 \hat{\alpha}_j^1}{\hat{\beta}_j^1 - \hat{\alpha}_j^1} ~~j\in\mathcal{I}
\end{cases}
\]
Since
\[
\hat{\alpha}_j^1
=
W_j^1\alpha^0
=
W_j^1x_0 - \epsilon \|W_j^1\|_1
\mbox{~~and~~}
\hat{\beta}_j^1
=
W_j^1\beta^0
=
W_j^1x_0 + \epsilon \|W_j^1\|_1
\]
We have
\[
\frac{\hat{\beta}_j^1}{\hat{\beta}_j^1 - \hat{\alpha}_j^1}\hat{\alpha}_j^1
=
\frac{(W_j^1x_0)^2 - (\epsilon \|W_j^1\|_1)^2}{2\epsilon \|W_j^1\|_1}
\]
Therefore, 
\begin{eqnarray*}
\E_{x_0} \min_{\alpha^0\preceq x\preceq\beta^0} T_L^{1:m}(x)_i-\min_{\alpha^0\preceq x\preceq\beta^0} \tilde{f}_i^{1:m}(x) \\
\le
\E W_i^{2:m} (D^1\hat{\alpha}^1 - [\hat{\alpha}^1]_{+})
=
\sum_j \E W_{ij}^{2:m} (D_j^1\hat{\alpha}^1 - [\hat{\alpha}_j^1]_{+}) \\
=
\sum_j \E r^{m-1}d^{m-2}(\vec{1}\vec{1}^T)_{ij}(D_j^1\hat{\alpha}^1 - [\hat{\alpha}_j^1]_{+})
\text{~~~ (By Lemma~\ref{lemma:approx-error-prop-from-first-layer}) } \\
=
\sum_j r^{m-1}d^{m-2} \E (D_j^1\hat{\alpha}^1 - [\hat{\alpha}_j^1]_{+}) \\
=
\sum_j r^{m-1}d^{m-2} \E\big(D_j^1\hat{\alpha}^1 - [\hat{\alpha}_j^1]_{+} | j\in\mathcal{I} \big)\p(j\in\mathcal{I}) \\
=
\sum_j r^{m-1}d^{m-2} \E\big(\frac{(W_j^1x_0)^2 - (\epsilon \|W_j^1\|_1)^2}{2\epsilon \|W_j^1\|_1}| j\in\mathcal{I} \big)\p(j\in\mathcal{I})
\end{eqnarray*}
By the same argument as in proof of statement 1 of Proposition~\ref{prop:single_layer_lower_bound}, we get
\begin{eqnarray*}
\E\big(\frac{(W_j^1x_0)^2 - (\epsilon \|W_j^1\|_1)^2}{2\epsilon \|W_j^1\|_1}| j\in\mathcal{I} \big)
\le 
\frac{\|W_j^1\|_2^2 - (\epsilon \|W_j^1\|_1)^2}{2\epsilon \|W_j^1\|_1}
\end{eqnarray*}
Since
$\|W_j^1\|_2^2 = r^2 d$ and $\|W_j^1\|_1^2 = (rd)^2$, we get
\[
\E\big(\frac{(W_j^1x_0)^2 - (\epsilon \|W_j^1\|_1)^2}{2\epsilon \|W_j^1\|_1}| j\in\mathcal{I} \big)
\le 
\frac{r^2 d - \epsilon^2 r^2d^2}{2\epsilon rd}
\]
We bound $\p(j\in\mathcal{I})$ using the same argument as in statement 1 of Proposition~\ref{prop:single_layer_lower_bound}: By property of Gaussian distribution, we can explicitly represent $\p(j\in\mathcal{I})$ as
\begin{eqnarray*}
\p (|W_j^1x_0| < \epsilon \|W_j^1\|_1)
=
\p (|\frac{W_j^1x_0}{\|W_j^1\|_2}| < \epsilon \frac{\|W_j^1\|_1}{\|W_j^1\|_2}) \\
=
2\Phi(\epsilon \frac{\|W_j^1\|_1}{\|W_j^1\|_2})-1
=
2 \Phi(\epsilon \frac{rd}{r\sqrt{d}}) - 1
=
2\Phi(\epsilon\sqrt{d})-1
\end{eqnarray*}
Combining this with the equality earlier, we get
\begin{eqnarray*}
\E_{x_0} \min_{\alpha^0\preceq x\preceq\beta^0} T_L^{1:m}(x)_i-\min_{\alpha^0\preceq x\preceq\beta^0} \tilde{f}_i^{1:m}(x) \\
\le 
\sum_{j=1}^d r^{m-1} d^{m-2} \frac{r^2 d - \epsilon^2r^2d^2}{2\epsilon rd} (2\Phi(\epsilon\sqrt{d}) - 1) \\
=
\sum_{j=1}^d r^{m-1} d^{m-2} (\frac{r}{2\epsilon}-\frac{\epsilon rd}{2}) (2\Phi(\epsilon\sqrt{d}) - 1) \\
=
(2\Phi(\epsilon \sqrt{d}) - 1) r^m d^{m-1} (\frac{1}{2\epsilon} - \frac{\epsilon d}{2}) \\
=
(\Phi(\epsilon \sqrt{d}) - 1/2) r^m d^{m-1} (\frac{1}{\epsilon} - \epsilon d)
\end{eqnarray*}
Since $\epsilon > \sqrt{1/d}$, which implies $\epsilon d - \frac{1}{\epsilon} > 0$, and since $m$ is chosen so that 
\[
m > \frac{1}{1-p}\log_d \frac{Bd}{(\Phi(\epsilon \sqrt{d}) - 1/2) (\epsilon d - \frac{1}{\epsilon})}
\]
It can be derived that
\[
(\Phi(\epsilon \sqrt{d}) - 1/2) r^m d^{m-1} (\epsilon d - \frac{1}{\epsilon})
>
B
\]
And we may conclude that
\[
\E_{x_0} \min_{\alpha^0\preceq x\preceq\beta^0} T_L^{1:m}(x)_i-\min_{\alpha^0\preceq x\preceq\beta^0} \tilde{f}_i^{1:m}(x)
\le 
-B
\]
\end{proof}
\begin{lemma}[``Zero-width'' tube on non-negative matrices]
\label{lemma:zero-width}
Suppose $T_L^{1:m}, T_U^{1:m}$ are lower and upper linear approximation functions to $\tilde{f}^{1:m}$ with network weights $W_{i, j}^{\ell}\ge 0$ and $b_i^{\ell}=0$, for any $i, j\in [d]$, $\forall\ell \in [m-1]$.
Then for any $2\le\ell\le m-1$, $b_U^{\ell} = b_L^{\ell} = 0$.
\end{lemma}
\begin{proof}
Fix any $\alpha^{\ell-1}, \beta^{\ell-1}$ for $\ell \ge 2$, by IBP iteration,
\begin{eqnarray*}
\hat\alpha^{\ell} = W^{\ell}\beta^{\ell-1}+[W^{\ell}]_{+}(\alpha^{\ell}-\beta^{\ell})+b^{\ell} \\
=
[W^{\ell}]_{-}\beta^{\ell-1} + [W^{\ell}]_{+}\alpha^{\ell-1}
=
[W^{\ell}]_{+}\alpha^{\ell-1} ~~~(W_{ij}^{\ell}\ge 0) \\
\succeq
0
\end{eqnarray*}
where the last term $[W^{\ell}]_{+}\alpha^{\ell-1}$ is non-negative since for $\ell>1$, the bound vector
\[
\alpha^{\ell-1} 
= g^{\ell-1}(\hat\alpha^{\ell-1}) 
= [\hat\alpha^{\ell-1}]_{+}
\succeq 0 \, ,
\]
So its inner product with the non-negative matrix $[W^{\ell}]_{+}$ is a non-negative vector.
Thus, $\forall i \in [d_{\ell}], \hat\alpha_i^{\ell}\ge 0$, and by definition all output nodes at layer $\ell$ satisfies $i \in \mathcal{I}_{+}$ (i.e., $0 \le \hat\alpha_i^{\ell} \le \hat\beta_i^{\ell}$, for all $i\in [d_{\ell}]$), and thus
\[
b_U^{\ell} = b_L^{\ell} = 0
\]
\end{proof}
\begin{lemma}
\label{lemma:approx-error-prop-from-first-layer}
Suppose $T_L^{1:m}, T_U^{1:m}$ are lower and upper linear approximation functions to $\tilde{f}^{1:m}$ with network weights $W_{i, j}^{\ell}\ge 0$ and $b_i^{\ell}=0$, for any $i, j\in [d]$, $\forall\ell \in [m-1]$. Then
\begin{enumerate}
    \item $T_L^{1:m}(x) = W^{1:m}x$.
    \item In addition, suppose $\forall \ell \in [m], \forall i,j \in [d], W_{ij}^{\ell}=r$. Then 
    \[
    W^{1:m}=W^{2:m}D^1W^1 ~~~\mbox{where}~~
    W^{2:m}=r^{m-1} d^{m-2} \vec{1}\vec{1}^T
    \]
\end{enumerate}
\end{lemma}
\begin{proof}
\textbf{Proof of first statement:} 
By definition, $T_L^{1:m}(x)=W_L^{1:m}(x)+b_L^{1:m}$ (see Definition~\ref{defn:multi-layer-lin-appx}). 
By Lemma~\ref{lemma:zero-width}, for any $2\le\ell\le m-1$, $b_U^{\ell}=b_L^{\ell}=0$, which implies that 
\begin{eqnarray*}
b_L^{1:m}
= 
b_L^{m:m} + \sum_{k=1}^{m-1} [W^{k+1:m}]_{-} b_U^k + [W^{k+1:m}]_{+} b_L^k \\
=
\sum_{k=1}^{m-1} [W^{k+1:m}]_{-} b_U^k + [W^{k+1:m}]_{+} b_L^k ~~~(\mbox{$b_L^{m:m} = b^m = 0$})\\
=
(\sum_{k=2}^{m-1} [W^{k+1:m}]_{-} b_U^k + [W^{k+1:m}]_{+} b_L^k) 
+ [W^{2:m}]_{-} b_U^1 + [W^{2:m}]_{+} b_L^1 \\
=
[W^{2:m}]_{-} b_U^1 + [W^{2:m}]_{+} b_L^1 ~~~(b_U^{\ell}=b_L^{\ell}=0,  \forall 2\le\ell\le m-1)
\end{eqnarray*}
On the other hand, $W^{2:m}$ is product of non-negative matrices, so $W^{2:m}_{ij}\ge 0$ and thus
\[
[W^{2:m}]_{-} b_U^1 + [W^{2:m}]_{+} b_L^1
=
[W^{2:m}]_{+} b_L^1
= 
[W^{2:m}]_{+} 0 = 0
\]
This shows that $T_L^{1:m}(x) = W_L^{1:m}x = W^{1:m}x$.

\textbf{Proof of second statement:}
Consider $D_{ii}^{\ell}$ for any $\ell \in \{2,\dots, m\}$, which is determined by the pre-activation IBP bounds $\hat\alpha^{\ell}_i, \hat\beta^{\ell}_i$:
\[
\hat\alpha_i^{\ell} = W_i^{\ell}\beta^{\ell-1} + [W_i]_{+}(\alpha^{\ell-1}-\beta^{\ell-1})
=
[W_i]_{-}\beta^{\ell-1} + [W_i]_{+}\alpha^{\ell-1}
=
[W_i]_{+}\alpha^{\ell-1}
\]
and
\[
\hat\beta_i^{\ell} = W_i^{\ell}\beta^{\ell-1} + [-W_i]_{+}(\beta^{\ell-1}-\alpha^{\ell-1})
=
[W_i]_{+}\beta^{\ell-1} + [W_i]_{-}\alpha^{\ell-1}
=
[W_i]_{+}\beta^{\ell-1}
\]
Since for any $\ell-1\ge 1$, the vectors $\alpha^{\ell-1}, \beta^{\ell-1}$ obtained from IBP iteration are non-negative valued due to the ReLU operation (i.e., $\alpha^{\ell-1}=g^{\ell-1}(\hat\alpha^{\ell-1})\succeq 0$ and $\beta^{\ell-1}=g^{\ell-1}(\hat\beta^{\ell-1})\succeq 0$), we get
\[
\hat\alpha_i^{\ell}
=
[W_i]_{+}\alpha^{\ell-1}
\ge 0
\mbox{~~and~~}
\hat\beta_i^{\ell}
=
[W_i]_{+}\beta^{\ell-1}
\ge 0
\]
Therefore, we can conclude that for any $\ell \ge 2$, $0\le \hat\alpha_i^{\ell}\le \hat\beta_i^{\ell}$, and by definition
\[
D_{ii}^{\ell} = 1, ~~~ \forall i \in [d] \implies D^{\ell} = I
\]
And therefore,
\[
W^{1:m} = W^{2:m}D^1W^1
\mbox{~~where~~}
W^{2:m} = W^m W^{m-1} W^{m-2} \dots W^2
\]
By the assumption that $W_{ij}^{\ell} = r, \forall \ell \in [m]$, $W^{\ell}=r \vec{1}\vec{1}^T$. So
\[
W^m W^{m-1} W^{m-2} \dots W^2
=
\underbrace{r \vec{1}\vec{1}^T r \vec{1}\vec{1}^T \dots r \vec{1}\vec{1}^T}_{\text{A product of $m-1$ terms}} 
=
r^{m-1} d^{m-2} \vec{1}\vec{1}^T
\]
\end{proof}
%
%
\section{Proof of Proposition~\ref{prop:unrolling-hurts}}
\unrolling*
\begin{proof}
Let $f^{1:m}$ be a network as specified in the second statement of Proposition~\ref{prop:transition}. Temporarily let $W^{k:m}$ denote the product of matrices $W^m W^{m-1}\dots W^k$. For any $i$,
\[
\min_{x} T_L^{1:m}(x)_i 
= W^{2:m}D^1W^1\alpha^0
\]
and since for $k-1 \ge 1$, $\alpha^{k-1}\succeq 0$, we get $D^k = I$ and
\[
\min_x T_L^{k:m}(x)_i
= W^{k+1:m}D^kW^k\alpha^{k-1}
= W^{k+1:m}W^k\alpha^{k-1}
= W^{k:m}\alpha^{k-1}
\]
Therefore,
\[
\min_{x} T_L^{1:m}(x)_i
-
\min_x T_L^{k:m}(x)_i
=
W^{k:m} (W^{2:k-1}D^1W^1\alpha^0 - \alpha^{k-1})
\]
Since for all $i$,
\[
\hat{\alpha}_i^{k-1}
=
W_i^{k-1}\beta^{k-2} + [W_i^{k-1}]_{+} (\alpha^{k-2} - \beta^{k-2})
=
W_i^{k-1}\alpha^{k-2}
\]
So for $k-2\ge 1$,
\[
\alpha^{k-1}
=
[W^{k-1}\alpha^{k-2}]_{+}
=
W^{k-1}\alpha^{k-2}
\]
Applying this recursively, we get
\[
\alpha^{k-1} = W^{k-1} \dots W^2 \alpha^1 = W^{2:k-1}\alpha^1
\]
Plug this into the inequality earlier, we get
\begin{eqnarray*}
\min_{x} T_L^{1:m}(x)_i
-
\min_x T_L^{k:m}(x)_i
=
W^{k:m} (W^{2:k-1}D^1W^1\alpha^0 - W^{2:k-1}\alpha^1) \\
=
W^{2:m} (D^1W^1\alpha^0 - \alpha^1)
=
W^{2:m} (D^1W^1\alpha^0 - [\hat{\alpha}^1]_{+})
\end{eqnarray*}
We may conclude the proof by noting that we can apply the same argument as that in Proposition~\ref{prop:transition}. The existence of $(x_0, \epsilon)$ can be shown by drawing $x_0\sim\mathcal{N}(0, I)$ and since by Proposition~\ref{prop:transition}, with appropriate choice of $m$ as a function of $\epsilon, d, B$, the expected gap $\min_{x} T_L^{1:m}(x)_i
-
\min_x T_L^{k:m}(x)_i$ is bounded by $B$. This implies that there exists $x_0$ in the sample space so that the corresponding $(\alpha^0, \beta^0)$ satisfies 
$\min_{x\in (\alpha^0,\beta^0)} T_L^{1:m}(x)_i
-
\min_{x\in (\alpha^k, \beta^k)} T_L^{k:m}(x)_i
< -B $.
\end{proof}
\section{Exact solution to hyper-rectangular constrained linear program}
\label{appdendix:ibp_derivation}
In our analysis, we frequently use the exact solution to the hyper-rectangular constrained linear optimization problem (as formulated in Eq~\eqref{eqn:ibp_lo} and \eqref{eqn:ibp_hi}). Here, we show the derivation of our exact solution.
We re-state the hyper-rectangular constrained linear optimization objectives:
\begin{eqnarray*}
\widehat{\alpha}^{\ell}_i := \min_{\alpha^{\ell-1} \preceq x \preceq \beta^{\ell-1}} e_i^T (W^{\ell} x + b^{\ell})\nonumber\\
\widehat{\beta}^{\ell}_i := \max_{\alpha^{\ell-1} \preceq x \preceq \beta^{\ell-1}} e_i^T (W^{\ell} x + b^{\ell})
\end{eqnarray*}
We derive the exact expression of $\widehat{\alpha}^{\ell}_i$ here (the derivation of $\hat\beta_i^{\ell}$ is similar). To solve the linear program, we introduce two Lagrange multipliers $\mu, \lambda \succeq 0$, and define Lagrangian of the original objective (readers may refer to \cite{boyd_vandenberghe_2004} for more background knowledge.)
\begin{eqnarray*}
L(x, \mu, \lambda)
=
W_i^{\ell} x + b_i^{\ell} + \mu^T (\alpha^{\ell-1}-x) + \lambda^T (x - \beta^{\ell-1})
\end{eqnarray*}
Via the Lagrangian and the Lagrange multipliers, we define the dual LP function $g(\mu, \lambda)$ as
\begin{eqnarray*}
g(\mu, \lambda) := \min_x L(x, \mu, \lambda) 
=
\min_x (W_i^{\ell} - \mu^T + \lambda^T)x + b_i^{\ell} + \mu^T\alpha^{\ell-1} - \lambda^T\beta^{\ell-1}
\end{eqnarray*}
The dual LP objective is
\[
\max_{\mu, \lambda} g(\mu, \lambda)
\]
Since the term $(W_i^{\ell} - \mu^T + \lambda^T)x$ is linear in $x$, its minimum becomes unbounded if $W_i^{\ell} - \mu^T + \lambda^T \ne 0$. The dual variables $\mu, \lambda$ are only feasible if $W_i^{\ell} - \mu^T + \lambda^T = 0$ (see \cite[Chapter 5]{boyd_vandenberghe_2004} for detailed introduction of Lagrangian, Lagrange multipliers, and duality). 
So we can add these constraints to the dual LP objective 
\[
\lambda \succeq 0
\mbox{~~and~~}
W_i^{\ell} + \lambda^T \succeq 0
\]
Utilizing these constraints, we can simplify the dual LP objective as
\begin{eqnarray*}
\max_{\mu, \lambda} b_i^{\ell} + \mu^T\alpha^{\ell-1} - \lambda^T\beta^{\ell-1}
=
\max_{\lambda} b_i^{\ell} + (W_i^{\ell} + \lambda^T)\alpha^{\ell-1} - \lambda^T\beta^{\ell-1} \\
=
\max_{\lambda} b_i^{\ell} + W_i^{\ell}\alpha^{\ell-1} + \lambda^T (\alpha^{\ell-1} - \beta^{\ell-1})
\end{eqnarray*}
Since $\alpha^{\ell-1} - \beta^{\ell-1} \preceq 0$, the constraints on $\lambda$ implies that
\[
\lambda_j =
\begin{cases}
- W_{i, j}^{\ell} \mbox{~~if~~} W_{i, j}^{\ell} < 0 \\
0 \mbox{~~o.w.}
\end{cases}
\]
So 
\[
\max_{\lambda} \lambda^T (\alpha^{\ell-1} - \beta^{\ell-1})
=
[-W_i^{\ell}]_{+} (\alpha^{\ell-1} - \beta^{\ell-1})
\]
Thus,
\begin{eqnarray*}
\widehat{\alpha}^{\ell}_i
=
\max_{\lambda} g(\lambda) 
=
b_i^{\ell} + W_i^{\ell} \alpha^{\ell-1} + [-W_i^{\ell}]_{+} (\alpha^{\ell-1} - \beta^{\ell-1})  \\
=
b_i^{\ell} + W_i^{\ell} \alpha^{\ell-1} - [W_i^{\ell}]_{-} (\alpha^{\ell-1} - \beta^{\ell-1}) \\
=
b_i^{\ell} + W_i^{\ell} \beta^{\ell-1} - [W_i^{\ell}]_{+} \beta^{\ell-1} + W_i^{\ell}\alpha^{\ell-1} - [W_i^{\ell}]_{-} \alpha^{\ell-1} \\
=
b_i^{\ell} + W_i^{\ell} \beta^{\ell-1} + [W_i^{\ell}]_{+} (\alpha^{\ell-1} - \beta^{\ell-1})
\end{eqnarray*}
Similarly, it can be derived that
\[
\widehat{\beta}^{\ell}_i
=
b_i^{\ell} + W_i^{\ell} \beta^{\ell-1}  + [-W_i^{\ell}]_{+} (\beta^{\ell-1} - \alpha^{\ell-1})
\]
In vector form, they can be written as
\begin{eqnarray*}
\label{eqn:ibp_preact}
\hat\alpha^{\ell} = W^{\ell}\beta^{\ell-1} + [W^{\ell}]_{+} (\alpha^{\ell-1}-\beta^{\ell-1}) + b^{\ell} \nonumber\\
\hat\beta^{\ell} = W^{\ell}\beta^{\ell-1} + [-W^{\ell}]_{+} (\beta^{\ell-1}-\alpha^{\ell-1}) + b^{\ell}
\end{eqnarray*}
In what follows, we show that when $(W^{\ell}, b^{\ell})$ are network weights at layer $\ell$ and when the input hyper-rectangular constraint covers the output space of layer $\ell-1$, then Eq~\eqref{eqn:ibp_preact} coincides with the pre-activation IBP bound in \cite{gowal:19}.
\begin{claim}[bound equivalence]
Given the same input bounds $\alpha, \beta$ and fix a network layer with general affine operation represented by $(W, b)$. Let $\underline{z}$ and $\bar{z}$ denote the two vectors representing the pre-activation IBP lower and upper bounds as defined in Eq~(6) \cite{gowal:19}, and let $\alpha^{\prime}, \beta^{\prime}$ denote the two vectors obtained using updates Eq~\eqref{eqn:ibp_lo} and Eq~\eqref{eqn:ibp_hi}. Then
$$
\underline{z} = \alpha^{\prime}
\text{~~and~~}
\beta^{\prime} = \bar{z} 
$$ 
\end{claim}
\begin{proof}
So we have $W = [W]_{+} + [W]_{-}$.
We consider $\alpha^{\prime} - \underline{z}$:
\begin{eqnarray*}
\alpha^{\prime} - \underline{z}
= W \beta + [W]_{+} (\alpha - \beta) - (W\frac{\beta+\alpha}{2} - |W|\frac{\beta-\alpha}{2}) \\
= W \frac{\beta - \alpha}{2} + [W]_{+} (\alpha - \beta) + ([W]_{+} - [W]_{-})\frac{\beta-\alpha}{2} \\
= (W - [W]_{-}) \frac{\beta - \alpha}{2} + [W]_{+} \frac{\alpha - \beta}{2} \\
= ([W]_{+}+[W]_{-}- [W]_{-}) \frac{\beta - \alpha}{2} + [W]_{+} \frac{\alpha - \beta}{2} 
= 0
\end{eqnarray*}
And similar argument applies to show $\beta^{\prime} = \bar{z}$, which we skip here.
\end{proof}
\paragraph{Comment:} In terms of computation, the update form \cite[Eq~(6)]{gowal:19} is perhaps more advantageous since it involves two matrix-vector multiplications, whereas ours need three. Their absolute value operator can also be implemented by ReLU operators ($|W| = [W]_{+} + [-W]_{+}$), which makes the overall forward update supported by standard deep learning libraries and more friendly to optimization.

\end{document}